\documentclass{article}

\PassOptionsToPackage{numbers, compress}{natbib}
\usepackage{diagbox}
\usepackage[preprint]{neurips_2022}
\usepackage[utf8]{inputenc} % allow utf-8 input
\usepackage[T1]{fontenc}    % use 8-bit T1 fonts
\usepackage{url}            % simple URL typesetting
\usepackage{booktabs}       % professional-quality tables
\usepackage{amsfonts}       % blackboard math symbols
\usepackage{nicefrac}       % compact symbols for 1/2, etc.

\usepackage{microtype}      % microtypography
\usepackage{amssymb}
\usepackage{amsmath}
\usepackage{amsfonts}
\usepackage[pdftex]{graphicx}
\usepackage{subcaption}
\usepackage{wrapfig}
\usepackage[ruled,vlined]{algorithm2e}

\usepackage{microtype}
\usepackage{booktabs} % for professional tables

\usepackage{tabularx}
\usepackage{wrapfig}
\usepackage[utf8]{inputenc} % allow utf-8 input
\usepackage[T1]{fontenc}    % use 8-bit T1 fonts
\usepackage{url}            % simple URL typesetting
\usepackage{amsfonts}       % blackboard math symbols
\usepackage{amsmath}
\usepackage{amssymb}
\usepackage{nicefrac}       % compact symbols for 1/2, etc.
\usepackage{microtype}      % microtypography
\usepackage{xcolor}
\usepackage{mathtools}
\usepackage{float}
\usepackage{multirow}
\usepackage{pifont}
\usepackage{multicol}
\usepackage{siunitx}
\usepackage{textcomp}
\usepackage{bm}
\usepackage{algorithm2e}
\usepackage{xcolor}
\DeclareMathOperator*{\argmax}{argmax}

\newcommand{\E}{\mathbb{E}}

\newcommand{\cA}{{\mathcal A}}

\newcommand{\cM}{{\mathcal M}}

\newcommand{\cO}{{\mathcal O}}
\newcommand{\cP}{{\mathcal P}}

\newcommand{\cX}{{\mathcal X}}

\newcommand{\bma}{{\bm a}}

\newcommand{\bmo}{{\bm o}}

\newcommand{\bms}{{\bm s}}

\newcommand{\bmx}{{\bm x}}
\newcommand{\bmy}{{\bm y}}

\newtheorem{example}{Example}

\newtheorem{thm}{Theorem}

\newtheorem{defn}{Definition}
\newenvironment{proof}{\paragraph{Proof:}}{\hfill$\square$}

\usepackage[backref=page]{hyperref}

\usepackage{minitoc}
\doparttoc % Tell to minitoc to generate a toc for the parts
\faketableofcontents % Run a fake tableofcontents command for the partocs

\title{Structured Q-learning For Antibody Design}
\author{%
  Alexander I. Cowen-Rivers  \\ Technische Universität Darmstadt 
\And
  Philip John Gorinski  \\ Huawei R\&D \And
  Aivar Sootla  \\ Huawei R\&D \And
  Asif Khan  \\ University of Edinburgh \And 
  Furui Liu  \\ Huawei R\&D
\And  Jun Wang  \\ Huawei R\&D  \\ University College London
\And Jan Peters \\ Technische Universität Darmstadt
\And Haitham Bou Ammar \\ Huawei R\&D  \\ University College London
}

\begin{document}
% \everypar{\looseness=-1}

\maketitle
\begin{abstract}
Optimizing combinatorial structures is core to many real-world problems, such as those encountered in life sciences. For example, one of the crucial steps involved in antibody design is to find an arrangement of amino acids in a protein sequence that improves its binding with a pathogen. Combinatorial optimization of antibodies is difficult due to extremely large search spaces and non-linear objectives. Even for modest antibody design problems, where proteins have a sequence length of eleven, we are faced with searching over $2.05 \times 10^{14}$ structures. Applying traditional Reinforcement Learning algorithms such as Q-learning to combinatorial optimization results in poor performance. We propose Structured Q-learning (SQL), an extension of Q-learning that incorporates structural priors for combinatorial optimization. Using a molecular docking simulator, we demonstrate that SQL finds high binding energy sequences and performs favourably against baselines on eight challenging antibody design tasks, including designing antibodies for SARS-COV. 
\end{abstract}
\section{Introduction}

% \begin{wrapfigure}{}{0.3\textwidth}
% %   \begin{center}
%     \includegraphics[width=0.35\textwidth]{figures/general/NPD.png}
% %   \end{center}
%     % \caption{Neural Protein Design illustration.~\footnote{Wet Lab image obtained from \url{https://edbiomed.ai/research/wetlab/}.}}
% \end{wrapfigure}
% Many fundamental combinatorial problems can be defined as optimization over combinatorial variables with respect to the objective function(s). 
% Combinatorial variables can be comprised of many complex data structures such as bit vectors, strings, ordinal and categorical variables. The goal of the optimization process is to find the optimal variables that minimise or maximise the objective(s). 
% However, the family of algorithms one would apply to solve this optimization problem depends on whether the objective function(s) are given as a mathematical expression e.g. defined as a polynomial of the input variables. This creates two fields of research, in combinatorial optimization over \textbf{unknown} objective(s) and combinatorial optimization over \textbf{known} objective(s).
Combinatorial optimization is a general problem faced in many domains, wherein one is tasked with finding an ordered or unordered placement of combinatorial variables with the goal of maximising an objective function.
Combinatorial problems with known objectives include maximum satisfiability (\textit{MaxSAT}) applications such as program synthesis~\cite{programsynthesis}, program verification~\cite{SilvaLM09} and automated theorem proving~\cite{z3solver, NIPS2017_b2ab0019, ntp2}. In engineering, combinatorial problems include finding optimal chip configurations\cite{otten08}, compiler optimization\cite{compiler} and logic synthesis\cite{boils}. In logistics, combinatorial problems include travelling salesman problems~\cite{bello2016neural,kool2018attention,vinyals2015pointer} and solving extremely large linear systems with constraints, commonly referred to as mixed-integer linear programming~\cite{BestuzhevaEtal2021OO, nair2020solving,balcan2018learning,khalil2016learning,gasse2019exact,bengio2020machine}. Lastly, in life sciences combinatorial problems involve core structures such as DNA, mRNA, and proteins, as well as functional molecules represented by SMILES~\cite{smilesrep} or SELFIES~\cite{SELFIES} strings. These fundamental combinatorial structures are central to medical breakthroughs such as the development of proteins to target cancer tumours~\cite{lds032}, as well developing mRNA vaccines for SARS-COV2~\cite{10665-338096}.\looseness=-1
\begin{figure}[th!]
    \centering
    \includegraphics[width=1.0\linewidth]{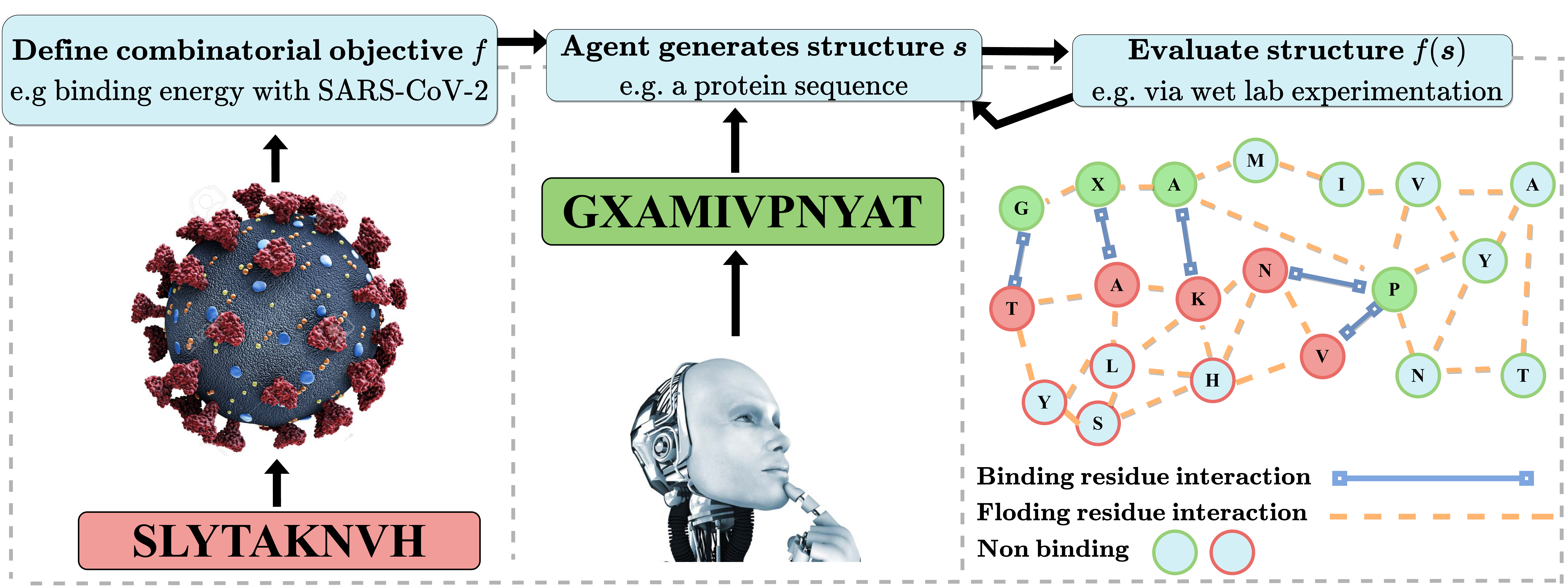}
    \caption{Illustrative example showing combinatorial optimization of an antibody using RL.}
    \label{fig:combtasks}
\end{figure}
% Note, when optimising a protein for specific property, the objective function involves an extremely costly evaluation through laborious wet lab experimentation. It is these expensive combinatorial objectives that are of particular interest to us.

\textbf{On the complexity of combinatorial optimization:} Combinatorial optimization is extremely challanging.
%, taking \textit{MaxSAT} as an example which is a well-studied family of combinatorial optimization problems within computational complexity theory; in fact many combinatorial optimization problems can be translated into \textit{MaxSAT}, such as graph colouring~\cite{inbook}. 
% The premise of \textit{MaxSAT} is that given the objective function's expression, one wants to determine if a solution exists (it is satisfiable) and if so, what the possible solutions are. Even In the simplest combinatorial optimization settings where each variable is Boolean (cardinality n=2), the practitioner may find the problem is indeed a
Taking \textit{MaxSAT} problems as an example, a well-studied family of combinatorial optimization problems within computational complexity theory,
they only deal with Booelan variables (with cardinality $n=2$), yet are proven to be NP-hard~\cite{CompleteCharacterization,cook1971}. Maximum satisfiability modulo theories (\textit{MaxSMT})~\cite{smtintro} extends \textit{MaxSAT} to handle complex combinatorial variables such as; bit vectors, strings, ordinal and categorical with higher cardinality ($n>>2$). \textit{MaxSMT} better resembles the optimization problems found within life sciences, as structures like DNA, proteins and SELFIES have corresponding cardinality $n=4,n=20,n\approx 500$. \textit{MaxSMT} solvers rely on quantifier elimination algorithms~\cite{Sturm2016SC2SC, smtcad,9700299,CowenRivers2018SummerRR} such as Cylindrical Algebraic Decomposition~\cite{Collins} (CAD). However, CAD's worst case complexity is \textbf{doubly exponential} in the number of variables~\cite{Collins,DAVENPORT198829,HEINTZ1983239} i.e. $\mathcal{O}(2^{2{^{L+1}}})$, highlighting the undesirable difficulty faced in combinatorial optimization \underline{\textit{even when the objective function is known}}.\looseness=-1
% APX-complete, meaning no polynomial-time approximation method exists (unless P=NP)
%In fact, \textit{MaxSAT} is proven to be even more challanging than a typical NP-hard task, as it is also \textbf{APX-complete}~\cite{CompleteCharacterization}.

% Without explicitly knowing the objective function's form, it can be hard to gauge the complexity of finding a solution. However, in the case where the objective function expression is known, \textit{MaxSMT} solvers can   determine possible solutions \textbf{if any solution exists}. To do this, 

\textbf{Motivation:} When there is nothing known about the objective function, one can only \textit{learn} about the inter-relationships between variables and their evaluation. Deep Reinforcement Learning (RL) has seen great success %learning
in many challanging applications, such as generating solutions to the NP-hard~\cite{Punn07} Travelling Salesman Problem~\cite{BelloPLNB16,kool2018attention} (TSP), GO~\cite{silver2016mastering} as well as controlling plasma~\cite{degrave2022magnetic}. Structural priors have been proposed for policy gradients~\cite{BelloPLNB16,kool2018attention} which improve combinatorial optimization, yet we have not seen structural priors developed for off-policy algorithms such as Q-learning.\looseness=-1

% In this work, we want to identify general improvements to get off policy methods to work at least as good as their on policy counterparts, and in doing so, we provide a general template for adapting off policy algorithms to combinatorial problems. We formulate the problem as a novel VAMP that we prove has the same optima as the underlying combinatorial optimization problem, and highlight SQL promising results in diverse set of experiments, which an emphasis on protein sequence optimization.\PJG{This transition sounds very abrupt: Nowhere before are we actually mentioning RL and on/off-policy methods, and ``SQL'' has not yet been introduced. Maybe re-formulate like below?}

The contributions of this work are three-fold: (i) We introduce \textit{Structured Q-learning (SQL)}, an extension of Q-learning equipped with structural priors such as structure critic targets, structure policy evaluation, structure exploration operator and structure policy improvement. (ii) Using a molecular docking simulator, we show the structural priors introduced allow SQL improves upon its unstructured and on-policy counterparts on a challenging suite of antibody design tasks.
% ii) We formalise combinatorial optimization as a markov decision process such that it allows \textit{sequential and non-sequential} structure generation, and prove it has the same optima as the underlying problem. 
\looseness=-1

% However, this is an extremely challenging optimization problem. Whereby a protein sequences which are themselves constructed through the compounding of (twenty) amino acids. Protein function is determined by the structure of the protein, where the placement of the amino acids in turn determines structure. With search spaces growing exponential with sequence length $L$ and number of variables $n$ $\mathcal{O}(n^L)$, one can readily notice with relatively small sequences $(L \geq 27)$, and few categorical variables (10), one quickly enter a search space with more possible sequence combinations than estimated atoms contained in the human body.

% \paragraph{Motivation } 
% Machine learning has made tremendous achievements such as AlphaFold2~\cite{} able to determine, with high accuracy, 3D strcuture of proteins - a grand challenge within structural biology.

% Therefore we have proof that current machine learning already has the ability to rival, and surpass, expert performance in some of the hardest challenges faced within science.

% \label{Sec:related}

\section{Related Work} 
% A long history exists for solving these combintorial optimization problems with \textbf{known objective functions}, particularly as \textit{SAT} was the first problem proved to be NP-complete~\cite{cook1971} and features in Karp's~\cite{karp1972reducibility} 21 NP-complete problems. Current state-of-art for binary variable (MaxSAT) optimization including~\cite{ignatiev,avellaneda2020short, martins2014open, lei2021cashwmaxsat, piotrow2020uwrmaxsat, morgado2019drmaxsat, nadel2019anytime}. Where many have approached MaxSAT with machine learning by learning heuristics for trandition solvers~\cite{satrl}. For more general combinatorial optimization over known objectives (MaxSMT), many algorithms exist~\cite{smtcad,9700299, borralleras2019incomplete,bofill2008barcelogic, BarrettCDHJKRT11} varying greatly depending on the problem. 

\textbf{The Machine Learning for Combinatorial optimization} (ML4CO)~\cite{ml4co} competition~\footnote{\href{https://www.ecole.ai/2021/ml4co-competition/}{https://www.ecole.ai/2021/ml4co-competition/}} was recently, hosted where competitors provide ML algorithms to improve the efficiency of solving Mixed Integer Programming (MIP) and Mixed Integer Nonlinear Programmng (MINLP) problems, compared to SOTA solvers such as SCIP~\cite{BestuzhevaEtal2021OO}. 
% They train AI based, usually black-box models on historical trajectories generated by expertise algorithms solving those problems, and extract knowledge so that the model can solve MIP\cite{bengio2020machine}.  
% Examples of the end-to-end learning model for problems like TSP, where the solutions of the instances are predicted by AI models, include  pointer network with a sequence-to-sequence architecture\cite{vinyals2015pointer}, which latter is showed to be trainable by Reinforcement Learning\cite{bello2016neural}. 
% Given the complexity of the MIP, researchers also explore the capability of various structures of the ML modules for capturing the intrinsic knowledge.  One piece of work shows that modules like attentions helps in performance\cite{kool2018attention}, which gives an obvious performance gain on  the vehicle routing problems.  
The combinatorial optimization approaches can be classified into two broad categories: those that fully rely on ML models to find a solution, and those that use ML to guide more traditional heuristic search algorithms. Under the algorithmic framework of solving the MIP, the methods in the later class use ML (including RL), to replace or aid a specific process so that the efficiency of problem solving is improved. \cite{DBLP:journals/corr/abs-2102-09663} tackles the problem of finding an initial feasible solution of a complex MIP by RL, in which CNN and MLP networks are used to encode the constraint matrix for informative state representations, and an efficient policy is designed. Other work~\cite{nair2020solving,balcan2018learning,khalil2016learning,gasse2019exact,bengio2020machine} focuses on learning-aided branch-and-bound, and train an RL policy by imitation learning on data obtained by applying a heuristic strong branching approach. Note, we do not compare to ML4CO methods, as they are tailored to MIP and are not applicable to our problem setting.\looseness=-1
% We do not compare to these ML/RL methods due to them being on a different problem, however we do include some of the popular search algorithms used as baselines e.g. simulated annealing~\cite{bertsimas1993simulated}, which is proven to converge to the global optima.

% Other forms of ML combinatorial optimization algorithms include mixed variable black-box optimization methods~\cite{cowen2020empirical,casmo}, 
Many neural combinatorial optimization methods require large datasets for pre-training, such as Pointer Networks~\cite{vinyals2015pointer} and latent space optimization~\cite{grosnit2021lsbo, griffiths2020constrained,gomez2018automatic, deshwal2021combining,maus2022local,daulton2021multi,tripp2020sample}. When a large, expertly obtained, dataset is not known upfront, methods typically resort to RL~\cite{BelloPLNB16,kool2018attention, ma2019combinatorial, cappart2021combinatorial,boffa2022neural, selsam2018learning, NIPS2014_a14ac55a} to iteratively learn from the objective function. Combinatorial optimization methods using RL are tailored to their application, such as the TSP~\cite{mazyavkina2021reinforcement,BelloPLNB16,kool2018attention}. When faced with optimising a TSP problem, one is given a set of nodes and must select the \emph{permutation} for visiting the nodes that leads to the shortest path. \cite{BelloPLNB16} first designed an RL method with Pointer Networks~\cite{vinyals2015pointer} that selects the permutation of nodes to visit based on a factorised policy. Follow up work~\cite{kool2018attention} made progress on improving the architecture with attention, as well as training methods such as greedy rollouts. 
% Due to the flexibility of not requiring an optimal dataset a-prior and success seen many domains, we choose to focus solely on RL for solving combinatorial optimization problems. 
Although we cannot directly apply these RL methods as our problem is not a permutation problem, we can construct baselines from their core components. \looseness=-1
% By conducting model training, efficiency of the algorithms is usually showed to with a higher level compared to the original heuristics on large scale problems. 

%see if more discussions are needed. ---Furui 
\textbf{Machine Learning for antibody design}
Several ML/computational approaches for antibody design~\citep{norman2020computational, akbar2021progress} have been developed, either using physics-based antibody and antigen structure modelling~\citep{fiser2003modeller,almagro2014second, leem2016abodybuilder} and docking~\citep{brenke2012application,sircar2010snugdock}, or using ML to learn the non-linear relationship between antibody-antigen binding directly from large sequence/structural datasets~\citep{akbar2021progress}. Methods 
such as~\citep{morea2000antibody,clark2006affinity, clark2009antibody, nimrod2018computational} assume the antigen-antibody is available and predict the affinity directly. 
Various recent generative models have proposed to generate antibody structure from their amino-acid sequences~\citep{amimeur2020designing,eguchi2020ig, shin2021protein,akbar2021silico,shuai2021generative,leem2021deciphering}. More recently~\citep{jin2021iterative} proposed an iterative refinement process to redesign the CDRH3 sequence of antibodies for improving properties such as the neutralising score. The caveat with these methods is that they require pre-training on large datasets. However, the process of collecting large antibody-antigen binding datasets is identified as a challenging problem due to the high cost of simulation~\cite{robert2021ymir,robert2021one,narayanan2021machine,laustsen2021animal}. In this work, we focus on the applicability of RL to designing antibodies from scratch, and thus do not compare to methods that require substantial pre-training.\looseness=-1

% Situations where the structure of an antibody is not known rely on sequence-to-structure methods to transform the protein sequence to a 3D molecular representation and then use docking software to compute the affinity with an antigen. 
% In a recent work~\citep{robert2021one,robert2021ymir} proposed the simulation suite Absolut!, an \textit{in silico} framework for simulating the antibody-antigen binding site. They show that findings in the Absolut! simulations transfer to the real world, demonstrating its importance as a simulation framework. In a very recent work~\cite{khan2022antbo} use combinatorial BO with Absolut! as an oracle for finding an optimal binding sequence. Simulation suites like Absolut! provide a challenging benchmark for the investigation and development of new ML methods for in silico antibody design. Here, we focus on the applicability of Reinforcement Learning to designing antibodies using the Absolut! docking simulator, and thus do not compare to the supervised or unsupervised learning methods that require substantial training data.

\section{Background}

We first define the general combinatorial optimization problem and introduce Reinforcement Learning as well as the Q-learning RL algorithm, which serves as the basis for Structured Q-learning.%\looseness=-1
%\subsection{Combinatorial optimization}\label{sec:opt}
% First we introduce key definitions for simplicity of describing combinatorial optimization variables collectively and independently.  
% \begin{defn}[\emph{Combinatorial variable}]\label{def:variables}
% Let $\mathcal{X}$ be a finite combinatorial set $\|\mathcal{X}\| < \infty$ then a combinatorial variable is defined as an element of a set i.e, $x \in \mathcal{X}$
% \end{defn}
% \begin{defn}[\emph{Combinatorial Structure}]
% A structure $s$ of size $L$ is an element from the set consisting of all $L$-tuples of combinatorial variables $s \in \mathbb{S} = \mathcal{X}^L$, where $L \in \mathbb{N}$.
% \end{defn}
\begin{defn}[\emph{Combinatorial optimization}]
Given a finite combinatorial set $\mathcal{X}$, a domain of structures $\mathbb{S}$ consisting of all $L$-tuples of combinatorial variables $\bms \in \mathbb{S} = \mathcal{X}^L$, where $L \in \mathbb{N}$, and an objective function $f:\mathbb{S} \rightarrow \mathbb{R}$, the goal of combinatorial optimization is to find the optimal structure $\bms^{\ast}$ that maximises $f$.
\begin{equation}\label{eq:combopt}
    \begin{array}{cc}
         \bms^{\ast}=\arg \max_{\bms \in\mathbb{S}} f(\bms)=\arg \max_{[\bmx_i]_{i=0}^{L-1}} f([\bmx_i]_{i=0}^{L-1})
    \end{array}
\end{equation}
\end{defn}
In the context of antibody design the objective function $f$ is a molecular docking simulator that takes an antibody protein sequence and evaluates its binding energy (affinity) towards a target antigen. Example~\ref{ex:1} illustrates the combinatorial structure of the antibody sequence space.
\begin{example}\label{ex:1}
For a set $\mathcal{X}$ of amino acids with $\|\mathcal{X}\|=20$ and a target protein of size $L=4$, the individual combinatorial variables are amino acids $x\in \mathcal{X}$, ordered to form a protein structure. Below example $s$ would constitute a valid structure over the combinatorial variables
\begin{equation}\label{example:protein}
    s=[x_i]_{0=1}^{L-1}=[x_0=A,x_1=H,x_2=D,x_3=W]=\text{AHDW}
\end{equation}
\end{example}
% However, these combinatorial optimization objective functions can be extremely expensive to evaluate (weeks to months), as they may require wet-lab experimentation. 
% \subsection{Reinforcement Learning}\label{sec:mdp}
% First we will provide background to RL, then we will formulate our combinatorial optimization problem in terms of RL.
\begin{defn}[\emph{Markov Decision Process}]
A Markov Decision Process (MDP) is defined as a tuple $\cM = \langle  \cO, \cA, \cP, r, \gamma_r \rangle$, where $\cO$ is a finite observation space; $\cA$ is a finite action space; $\gamma_r \in (0, 1)$ is the task discount factor; $\cP: \cO \times \cA \times \cO \rightarrow [0, 1]$, i.e., $\bmo_{t+1} \sim p(\cdot | \bmo_t , \bma_t)$; and $r:\cO \times \cA \rightarrow [0, +\infty)$ is the task reward. We can define the optimal value function by $V^\ast(\bmo) = \max_{\pi} \E^\pi_\bmo J_{\rm task}$, where $\E^\pi_{\bmo'}$ is the mean rewards over trajectories from policy $\pi$ starting at observations $\bmo'$. The goal of RL is to find the optimal policy $\pi^{\ast}$ of the MDP:
\begin{equation}\label{prob:vanilla_mdp}
\begin{aligned}
\pi^{\ast}=\max_\bma Q^{\ast}(\bmo,\bma) = \max_\bma \sum_{\bmy \in \mathbb{O}} \hat{\cP}(\bmy \mid \bmo, \bma)\left[r(\bmo, \bma, \bmy) + \gamma V^{\ast}(\bmy) \right]
\end{aligned}
\end{equation}
% \begin{equation}\label{prob:vanilla_mdp}
% \begin{aligned}
%         \pi^{\ast} = \arg \max_{\pi}~ \E^\pi_{\bmo'} \left[ J_{\rm task} \right], \quad J_{\rm task}\triangleq \sum\limits_{t=0}^\infty  \gamma_r^t r(\bmo_t, \bma_t),
%  \end{aligned}
% \end{equation}
% where actions are defined as $\bma_t \sim \pi$. 
\end{defn}

We develop a method that builds upon Q-learning~\cite{watkins1992q}, one of the most successful methods for finite action spaces. Q-learning works by first collecting \textit{Random Evaluations} of actions in the MDP, and then cycling between i) updating the policy using \textit{Critic Targets}, ii) \textit{Policy Evaluation}, and iii) \textit{Policy Improvement}, where the critic is updated with the following \textit{Critic Targets}, with learning rate $\alpha$;\looseness=-1
\begin{equation}\label{eq:qlerning}
    % \begin{array}{ll}
         Q_{k+1}(\bmo_t,\bma_t) \Longleftarrow  Q_k(\bmo_t,\bma_t) + \alpha (r(\bmo_t, \bma_t) + \gamma\max\limits_{\hat{\bma}}Q(\bmo_{t+1}, \hat{\bma}) -Q(\bmo_t,\bma_t))
    % \end{array}
\end{equation}

% where in \textit{Policy Evaluation} we select a new action with $\max\limits_{\hat{a}}Q(o_{t+1}, \hat{a})$, and in \textit{Policy Improvement} we decide to take either an exploitative or random action. 
%\subsection{Combinatorial optimization using MDP's}
% We want to define the MDP in a way such that we can construct combinatorial structures such as the protein in Example~\ref{ex:1}. 
% To develop an RL algorithm for \emph{combinatorial problems}, we need a formulation of MDP that can leverage the domain structure for sequential decision making. 
To develop an RL algorithm for flexible structural generation, we want to allow for both sequential generation (of potentially arbitrary length), as well as allocating variables simultaneously to a fixed-length unallocated sequence. Previous works attempt to define similar padded MDP~\cite{mazyavkina2021reinforcement}, however we believe our MDP formalism to be the first.\looseness=-1
% To do this, we must allow for both i) sequential transition as well as ii) a predefined unallocated sequence where the variables can be allocated simultaneously. 
% \AS{Here it could be good to explain how and why you define the MDP this way}
\begin{defn}[\emph{Variable Allocation Markov Decision Process}]
A variable allocation Markov Decision Process (VAMP) is defined as a tuple $\cM_{\text{CO}} = \langle  \hat{\cO}, \hat{\cA}, \hat{\cP}, \hat{r}, \gamma_{\text{CO}} \rangle$, with observations containing both combinatorial variables and masked variables $\hat{\mathbb{S}}= \hat{\cX}^L$, where $\hat{\cX} = \cX \cup {\textbf{MASK}}$, with $\mathbb{S} \subset \hat{\mathbb{S}}$, actions $\hat{\cA} =\cX$. We define a deterministic transition function given as: 
\begin{equation}
    \begin{array}{cc}
         \hat{\cP}\left(\bmo_t=[[\bma_i]_{i=0}^t, [\textbf{MASK}]_{t+1}^T],\bma_t, \bmo_{t+1}=[[\bma_i]_{i=0}^{t+1}, [\textbf{MASK}]_{t+2}^T]\right)=1
    \end{array}
\end{equation}
%The combinatorial optimization objective discount factor $\gamma_{\text{CO}}$.
The set of starting observations~\footnote{If there is context to the optimization problem, we can concentrate it to $\bmo_0$.} being $[\textbf{MASK}]_{i=0}^{T}$. Lastly, the reward function $\hat{r}$ is defined below, where $[\cdot]$ defines a concatenation operator.%\looseness=-1
% \begin{equation}
%     \begin{array}{cc}
%          \hat{r}(o_t,a_t,o_{t+1}) =\begin{cases}
%              0  & \text{if } t+1 < L   \\
%              f([a_i]_{i=0}^{L-1}) & \text{else }
%       \end{cases} 
%     \end{array}
% \end{equation}
\begin{equation}
    \begin{array}{cc}
         \hat{r}(\bmo_t,\bma_t,\bmo_{t+1}) =\begin{cases}
             0  & \text{if } t < L-1   \\
             f([\bma_i]_{i=0}^{L-1}) & \text{if } t = L - 1
       \end{cases} 
    \end{array}
\end{equation}
\end{defn}
% \AS{
% I believe your 
% }
% \begin{defn}
% We define a combinatorial optimization Markov Decision Process (VAMP) as a tuple $\cM_{\text{CO}} = \langle  \hat{\cO}, \hat{\cA}, \hat{\cP}, \hat{r}, \gamma_{\text{CO}} \rangle$, with observations $\hat{\cO} \subseteq \mathbb{S}$, actions $\hat{\cA} = \cX$, as defined in Definition~\ref{def:variables}. Transition probability density function $\hat{\cP}(o_t,a_t, o_{t+1}=[o_t,a_t]$~\footnote{We can simplify the expression $o_{t+1}=[o_t,a_t]=[a_i]_{i=1}^{t}$}$)=1$, such that $t \leq L$.  The combinatorial optimization objective discount factor $\gamma_{\text{CO}}$. The set of starting observations being the empty set $\emptyset$~\footnote{If there is context to the combinatorial optimization problem, one can simply use it in place of $\emptyset$.}, the size of the combinatorial structure being optimised $L$. Lastly, the reward function $\hat{r}$ is defined as;
% \begin{equation}
%     \begin{array}{cc}
%          \hat{r}(o_t,a_t,o_{t+1}) =\begin{cases}
%              0  & \text{if } t+1 < L   \\
%              f([a_i]_{i=0}^{L-1}) & \text{else }
%       \end{cases} 
%     \end{array}
% \end{equation}
% \end{defn}

% The transition function simply replaces the \textbf{MASK} token at position $t$ with $\bma_t$. We simplify notation by omitting the masked elements of the sequence e.g. $\bmo_T=[\bma_i]_{i=0}^T$, where $T < L$. 

\begin{example}
We can now show how the same protein from~\ref{example:protein} can be constructed under VAMP. 
\begin{equation}
    \begin{array}{ll}
        o_{4} = \hat{\cP}(\hat{\cP}(\hat{\cP}(\hat{\cP}(o_0=[\textbf{MASK}]_{0}^{3}], a_0=\text{A}),a_1=\text{H}),a_2=\text{D}),a_3=\text{W})=\text{AHDW}
    \end{array}
\end{equation}
\end{example}

Note, VAMP global optima is the same global optima of the underlying function (Theorem~\ref{thrm:1}). 

\begin{thm}\label{thrm:1}
    For combinatorial optimization with discount factor $\gamma_{\text{CO}}=1$ under any objective function $f$, the optimal policy $\pi^\ast$ finds a solution $\bms^\ast$ such that $f(\bms^\ast)  = \arg \max_{\bms \in\mathbb{S}} f(\bms)$. 
    % (See Appendix~\ref{thrm:1proof} for \textbf{Proof}).
    % finding the optimal policy $\pi^{\ast}$ is equivalent to finding the global optima $\bms^{\ast}$ of Equation~\ref{eq:combopt}.
\end{thm}
\begin{proof}
See Appendix~\ref{thrm:1proof}.
\end{proof}
\section{Structured Q-learning}\label{sec:sql} 

With the above relevant background, we introduce our proposed algorithm \textbf{Structured Q-learning}. SQL is an off-policy RL algorithm for combinatorial optimization that consists of four components that introduce structural priors: \textit{structure critic targets}, \textit{structure policy evaluation}, a \textit{structure exploration operator} $\phi$, and \textit{structure policy improvement}, as shown in Algorithm~\ref{alg:sql} in the Appendix.\looseness=-1

\textbf{Random Structure Evaluations}\label{sec:rstruct} In the first step, we sample random structures $\bms$ and evaluate them in the environment to obtain rewards for $f(\bms)$ for the full structures. The observed pairs $\{\bms^{(i)}, f(\bms^{(i)})\}$ are saved in a buffer (\text{StructBuffer}) and used as initial examples to train SQL's structure critics.\looseness=-1

\textbf{Structure Critic Targets}\label{sec:targets} 
% During structure critic training we assign a constant credit at each step the structure was generated, with value of the $f(s)$ (the reward achieved when the structure was evaluated in the objective function).
We train a structure critic on the observed structure-reward pairs obtained from the previous step. For a given sequence $\bms$, the critic's objective is to learn to predict the expected reward $f(\bms)$ directly. This step can be achieved similarly to the update for the factorised policy (shown in the Appendix in  Equation~\ref{eq:policyloss}), whereby the same training signal, that is, the true reward observed for the complete structure, is propagated at each step of combinatorial variable prediction:\looseness=-1
\begin{equation}\label{eq:critic_learn}
    \begin{array}{ll}
        %  & \mathcal{S}(o,a) \Longleftarrow  \mathcal{S}(o,a) + \alpha (\sum_{i=0}^L \gamma^i R(o_i, a_i) -\mathcal{S}(o,a)) 
        & \mathcal{S}([\bma_i]_{i=1}^t) \Longleftarrow  \mathcal{S}([\bma_i]_{i=1}^t) + \alpha \cdot (f([\bma_i]_{i=1}^L) -\mathcal{S}([\bma_i]_{i=1}^t))
        %  \\ & \mathcal{S}(s=a_{t<L}) \Longleftarrow  \mathcal{S}(s=a_{t<L}) + \alpha (f(s=a_{t<L}) -\mathcal{S}(s=a_{t<L})) 
    \end{array}
\end{equation}
% We remind the reader that $o_t=[a_i]_{i=1}^t$, $s=o_L=[a_i]_{i=1}^L$ and $r(o_L)=f(o_L)$.
% Note, this is no longer temporal difference learning but is in fact a monte-carlo method, as it requires return to be calculated before updates can be computed.
\textbf{Structure Policy Evaluation}\label{sec:spe}
A critic is trained to take structures and predict their objective function values; we can perform policy evaluations in order to determine what the exploitative (greedy) structure $\bms^\ast=\arg \max_{\bms \in\mathbb{S}} \mathcal{S}(\bms)$ should be for the next policy improvement step. For this purpose, we use the critic to generate a new structure used to evaluate the objective function. In general, generation can be done either sequentially or simultaneously.\looseness=-1
%We can choose to generate a new structure for evaluation in the objective function one of two ways, either Autoregressive (AR) generation or Non-autoregressive (NAR) generation. We investigate both sequential and non-sequential generation of the structure to solve this maximisation problem.

\emph{\textbf{Sequential Generation}} constructs the greedy sequence from the trained structure critic one step at a time, at each step choosing the output that maximises the sequence score up to this point:
\begin{equation}\label{eq:seq_gen}
    \begin{array}{cc}
     \argmax_{\bms \in\mathbb{S}} \mathcal{S}(\bms) = \argmax_{\bma_L \in \mathcal{X}} \mathcal{S}(\dots \argmax_{\bma_2 \in \mathcal{X}} \mathcal{S}([ \argmax_{\bma_1 \in \mathcal{X}} \mathcal{S}(\bma_1), \bma_2])\dots, \bma_L]) 
\end{array}
\end{equation}
% \begin{equation}
%     \begin{array}{cc}
%      s^\ast = \max_{s \in\mathbb{S}} \mathcal{S}(s) = \max_{a_L \in \mathcal{X}} \dots \max_{a_2 \in \mathcal{X}} \max_{a_1 \in \mathcal{X}} \mathcal{S}(\dots \mathcal{S}([\mathcal{S}(a_1), a_2])\dots, a_L]) 
% \end{array}
% \end{equation}
\begin{figure}[th!]
    \centering
    \includegraphics[width=1.0\linewidth]{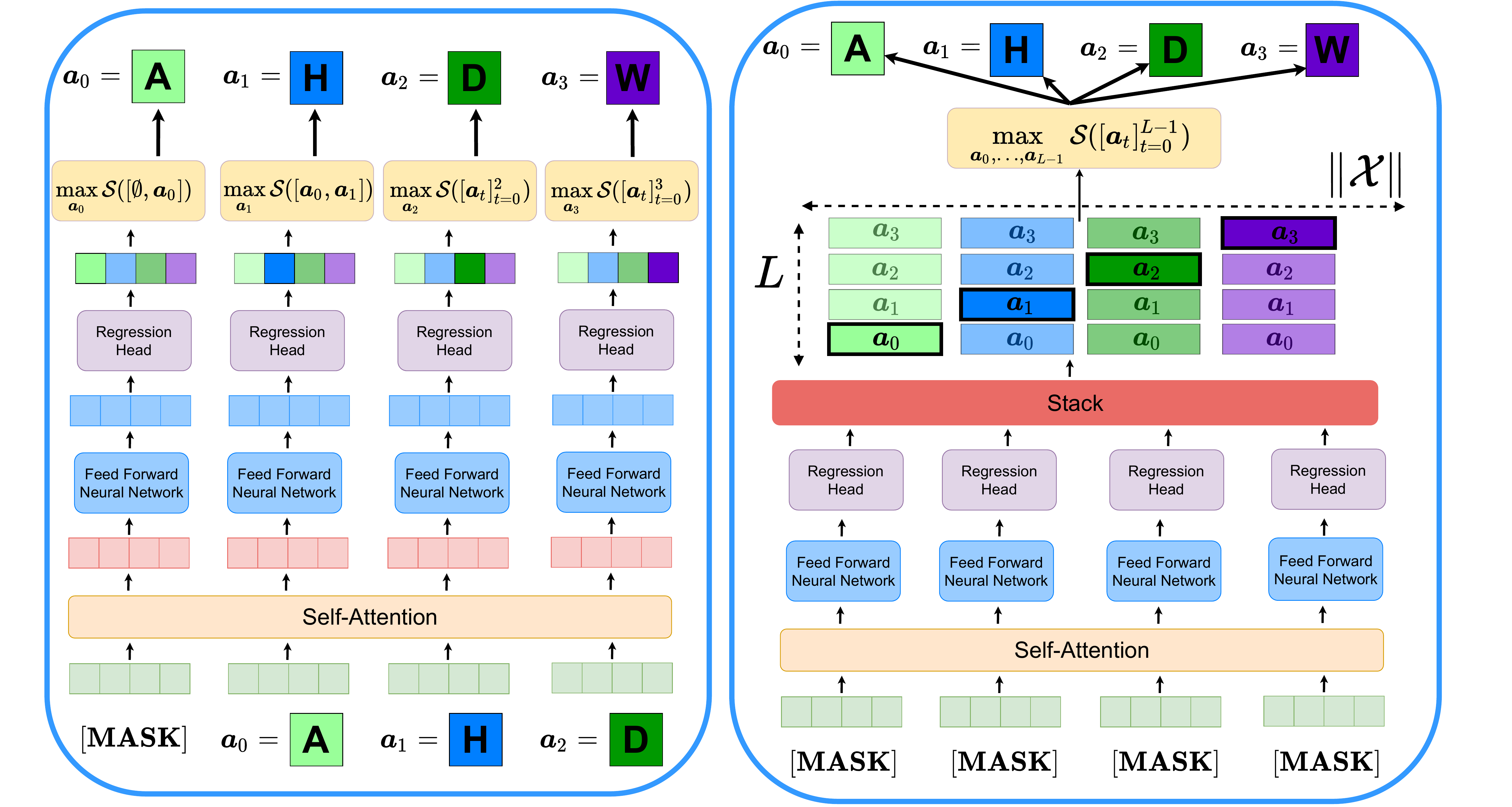}
    \caption{We illustrate two structure policy evaluation strategies implemented with a Transformer. A sequential strategy (Greedy) is shown on the left,  and a non-sequential strategy (Masked) shown on the right. Note, both structure policy evaluation strategies take the same input sequence $[\textbf{MASK}]_0^3$ and arrive at the same output sequence $\textbf{AHDW}$ through different computation graphs.}
    \label{fig:spi}
\end{figure}
\noindent\emph{\textbf{Simultaneous Generation}}, rather than iteratively performing step-wise policy evaluations, formulates the generation procedure to directly produce the highest-value sequence, given the current critic:
%Rather than iteratively performing step-wise policy evaluation in order to generate a full sequence of combinatorial variables, we can also formulate the generation procedure to directly generate the highest-value sequence, given the current critic.
\begin{equation}\label{eq:simul_gen}
    \begin{array}{cc}
     \arg \max_{\bms \in\mathbb{S}} \mathcal{S}(\bms) = \arg \max_{\bma_1 \in \mathcal{X}, \ldots, \bma_L \in \mathcal{X}} \mathcal{S}([\bma_1, \ldots, \bma_L])
\end{array}
\end{equation}
%We will discuss specific details for utilising both AR and NAR generation in Section~\ref{sec:speimp}.
% We investigate both sequential and simultaneous generation of structures to solve this maximisation problem in Section~\ref{sec:speimp}. 
The choice of using a Transformer~\cite{vaswani2017attention} network for the Structure Critic allows for straight-forward implementation of either generation method, illustrated in Figure~\ref{fig:spi}: In the sequential case, we train the SQL Critic in a similar fashion to autoregressive Language Models such as GPT\cite{brown2020language} with a diagonal attention mask, preventing ``lookahead'' in the attention mechanism. When aiming for simultaneous generation, the Critic is instead trained similarly to non-autoregressive Masked Language Models like BERT~\cite{devlin2018bert} or RoBERTa~\cite{liu2019roberta}, enabling it to predict multiple variables at the same time. \looseness=-1

%\looseness=-1

% \begin{equation*}
%     \begin{array}{cc}
%          &  [\text{BOS}], [\text{MASK}], [\text{MASK}], [\text{MASK}], [\text{MASK}], [\text{MASK}], [\text{MASK}], [\text{MASK}], [\text{MASK}], [\text{EOS}]
%     \end{array}
% \end{equation*}
% Then perform one forward pass through a transformer network to predict a matrix of $[L,20]$ (sequence length x amino acids). We can then take the arg max along the last dimension, which will give us $L$ greedy predicted amino acids (for each masked position in the protein we wanted to fill). In order to make this greedy sequence meaningful, when we perform substructure critic updates we can compute the mean squared error of the predicted $[L,20]$ matrix output with the rewards from a given sequence, indexing logits only at the evaluated amino acid locations. We do not use masking for the transformer models for this method, as we never a sequence generated sequentially (aka a substructure). Note the time complexity of this strategy is $\mathcal{O}(1)$, as it only requires a single forward pass through the substructure critic regardless of sequence length.
% When running an agent with any of these pure greedy strategies, we observe no improvement, which is of no surprise as we have only introduced pure exploitative policies -- that does not support learning. Of course we must use a exploration strategy, perhaps one more tailored for combinatorial optimization problems?

\textbf{Structure Exploration Operator $\bm{\Phi(\cdot)}$}\label{sec:operator} 
In order to explore new structures during training, we introduce the structure exploration operator $\Phi(\cdot)$. To acquire a random structure for exploration, we uniformly sample a structure $\bms^{(i)}$ from \text{StructBuffer}. We then apply the \textit{replace} operation, which uniformly chooses a substructure $\bmx_j$ and replaces it with another $\hat{\bmx}$ uniformly sampled from $\hat{\bmx}\sim\mathcal{X}$. Note, there are many other available implementations for the operator $\bm{\Phi(\cdot)}$ e.g. performing crossover between structures, however exploring all operators is out of scope for this work. We define exploitative \emph{structure} as $\bms^{\ast}$, and random (exploration) structure as $\hat{\bms}$. \looseness=-1

\textbf{Structural Policy Improvement}\label{sec:spi}
Policy improvement trades off when to explore or exploit. We can generalise these exploration strategies to structures. In order to do so, we define a criterion $p(\textbf{accept}=\bms^{\ast})$ for selecting either a greedy or a random structure. We propose 3 exploration strategies for $p(\textbf{accept}=\bms^{\ast})$, adapted to operate on structures.\looseness=-1

\emph{e-greedy:}
An epsilon-greedy (e-greedy) criterion for exploration entails accepting an exploitative action $\bma$ with probability $p(\textbf{accept}=\bma)=1-\epsilon_{g}(n)$, where $\epsilon_{g}(n) \in [\epsilon_{\text{min}}, 1]$ is a decaying function of steps $n$. For structures, we use the same criterion for accepting an exploitative \emph{structure} $\bms^{\ast}$.\looseness=-1

\emph{$\mathcal{S}$-greedy} We introduce a novel criterion based on a secondary structure critic ($\mathcal{S}_2(\cdot)$). We do so by comparing the predicted values of $\mathcal{S}_2$ for a random structure $\hat{\bms}$ and an exploitative structure $\bms^{\ast}$. If $\mathcal{S}_2(\bms^{\ast}) > \mathcal{S}_2(\hat{\bms})$ then $p(\textbf{accept}=\bms^{\ast})=1$, otherwise $p(\textbf{accept}=\bms^{\ast})=e^{\mathcal{S}_2(\hat{\bms})-\mathcal{S}_2(\bms^{\ast})}$.\looseness=-1

\emph{Sampling:} Rather than exploring greedily, we can sample actions from a stochastic policy with probability $p(\bma \mid \bmo) = \frac{\exp^{\mathcal{S}(\bmo,\bma)}}{\sum_{\hat{\bma} \in \mathcal{X}} \exp^{\mathcal{S}(\bmo,\hat{\bma})}}$. This will favour choosing structures with high rewards, but allow for exploration. We fix $p(\textbf{accept}=\bms^{\ast})=1$, as the stochastic policy takes care of exploration.\looseness=-1

% \begin{equation}
%     \begin{array}{cc}
%          &  p(a \mid o) = \frac{\exp^{\mathcal{S}(o,a)}}{\sum_{\hat{a} \in \mathcal{X}} \exp^{\mathcal{S}(o,\hat{a})}}
%     \end{array}
% \end{equation}

%\paragraph{Beam (Autoregressive):} For arbitrary beam width $k$, during sequential decoding one keeps the top $k$ structures with respect to the most recent predicted objective function value~\footnote{When $k=1$ this is equivalent to greedy decoding.}. 
% Note the time complexity of this decoding strategy is proportional to $L$ number of combinatorial variables $\mathcal{O}(L)$. 

%\paragraph{Masked (Non-autoregressive):} We can take inspiration from masked language modelling, and mask all variables of the target structure and predict their variable assignment simultaneously.
% This requires predicting an output matrix $L \times \| \mathcal{X} \|$, then taking the $\max$ combinatorial variable for each position in the sequence. 
% The time complexity of this strategy as it requires a single forward pass through the substructure critic is constant $\mathcal{O}(1)$.

%Here we introduce three additional exploration strategies adapted to operate on the whole structures, rather than individual actions.
\section{Experiments and Results}

\begin{figure}[th!]
    %  \vspace{-10pt}
    \centering
    \includegraphics[width=0.220\linewidth]{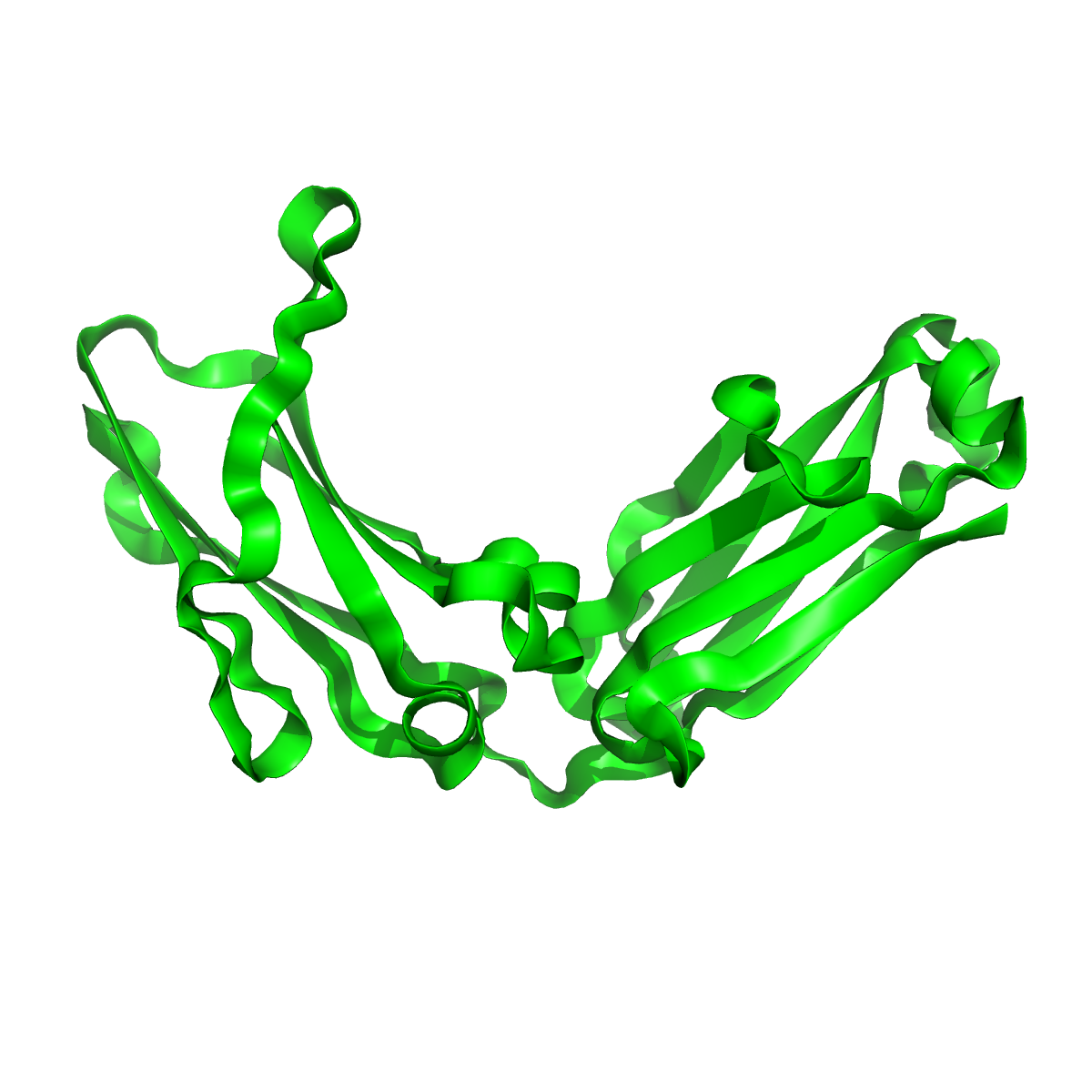}
    \includegraphics[width=0.220\linewidth]{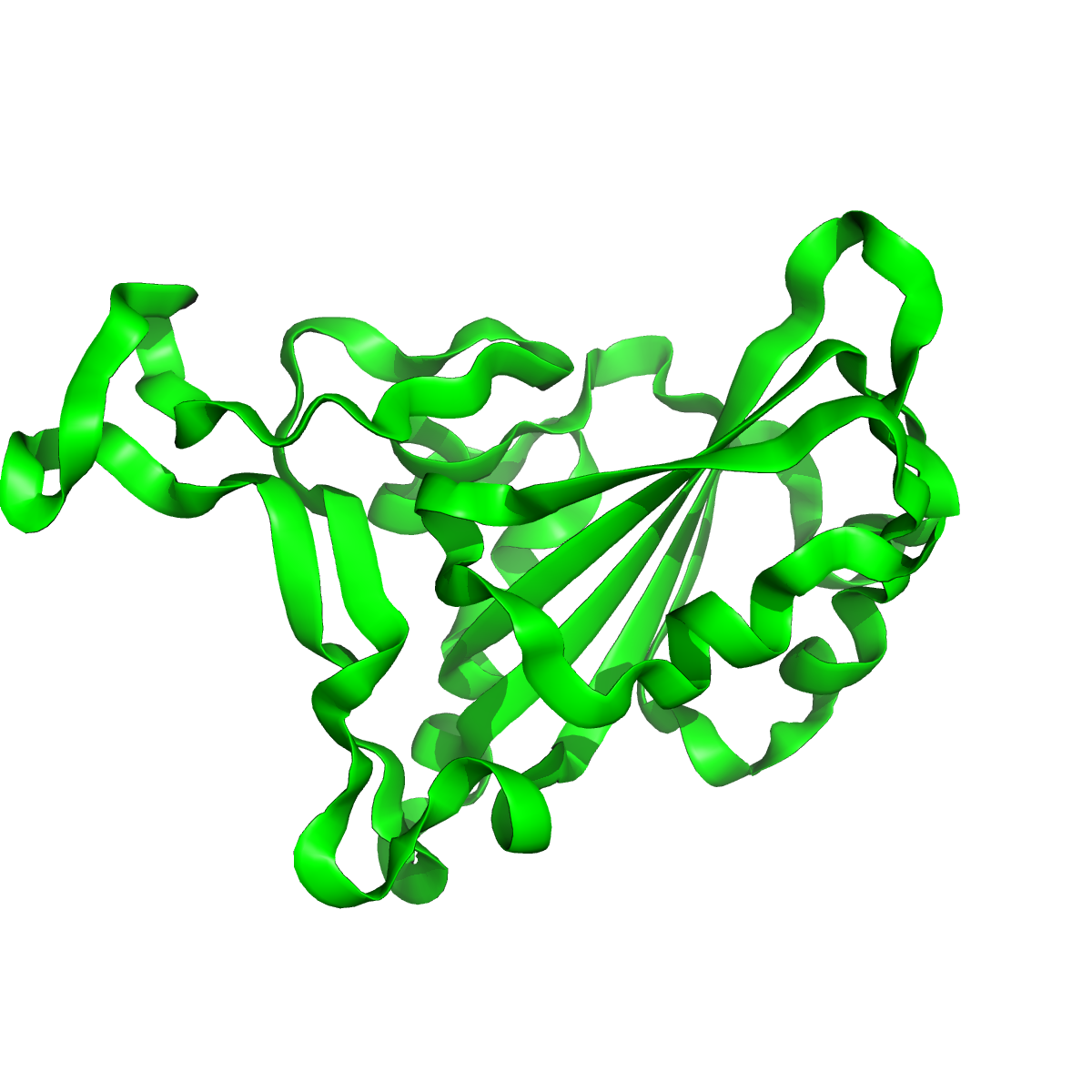}
    \includegraphics[width=0.220\linewidth]{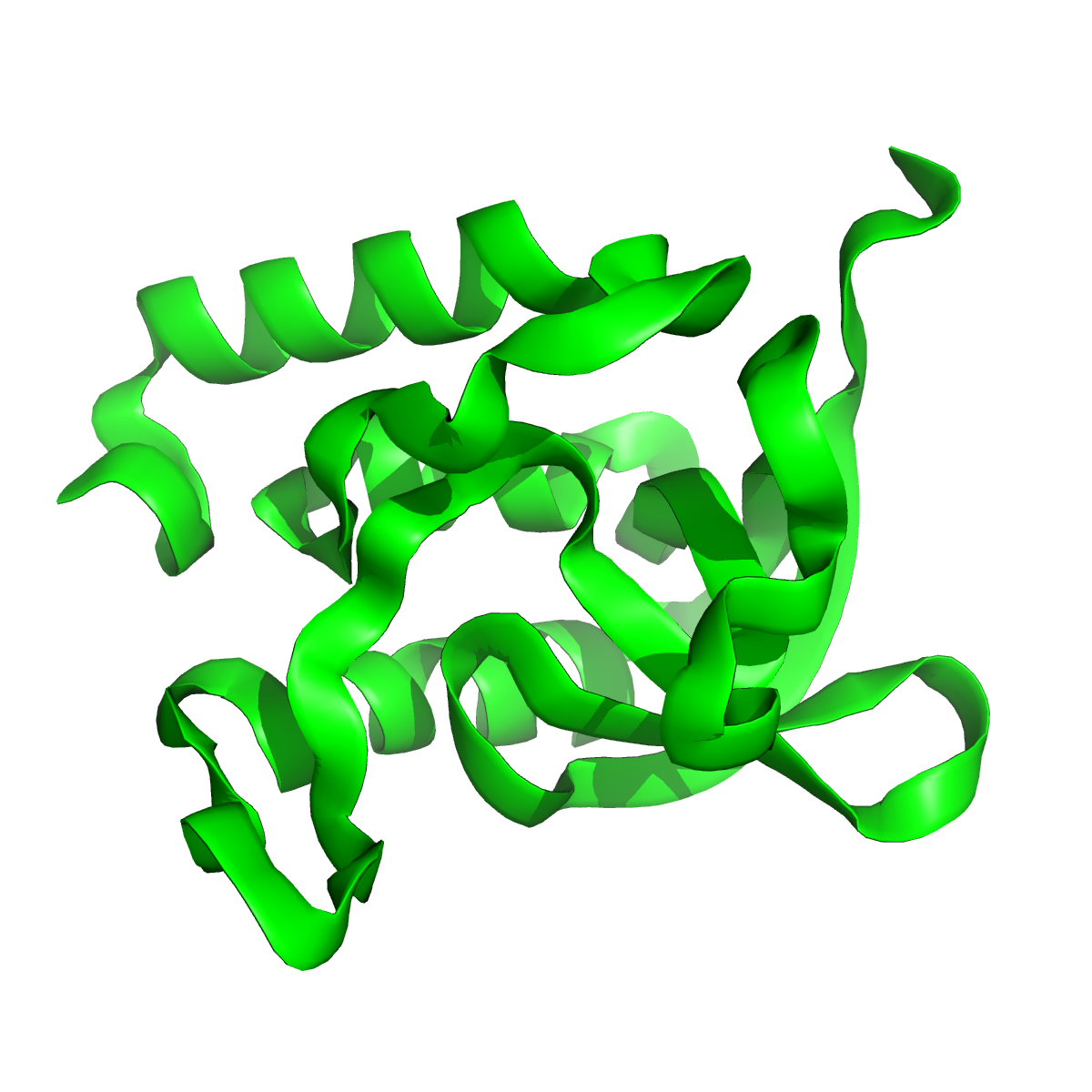}
    \includegraphics[width=0.220\linewidth]{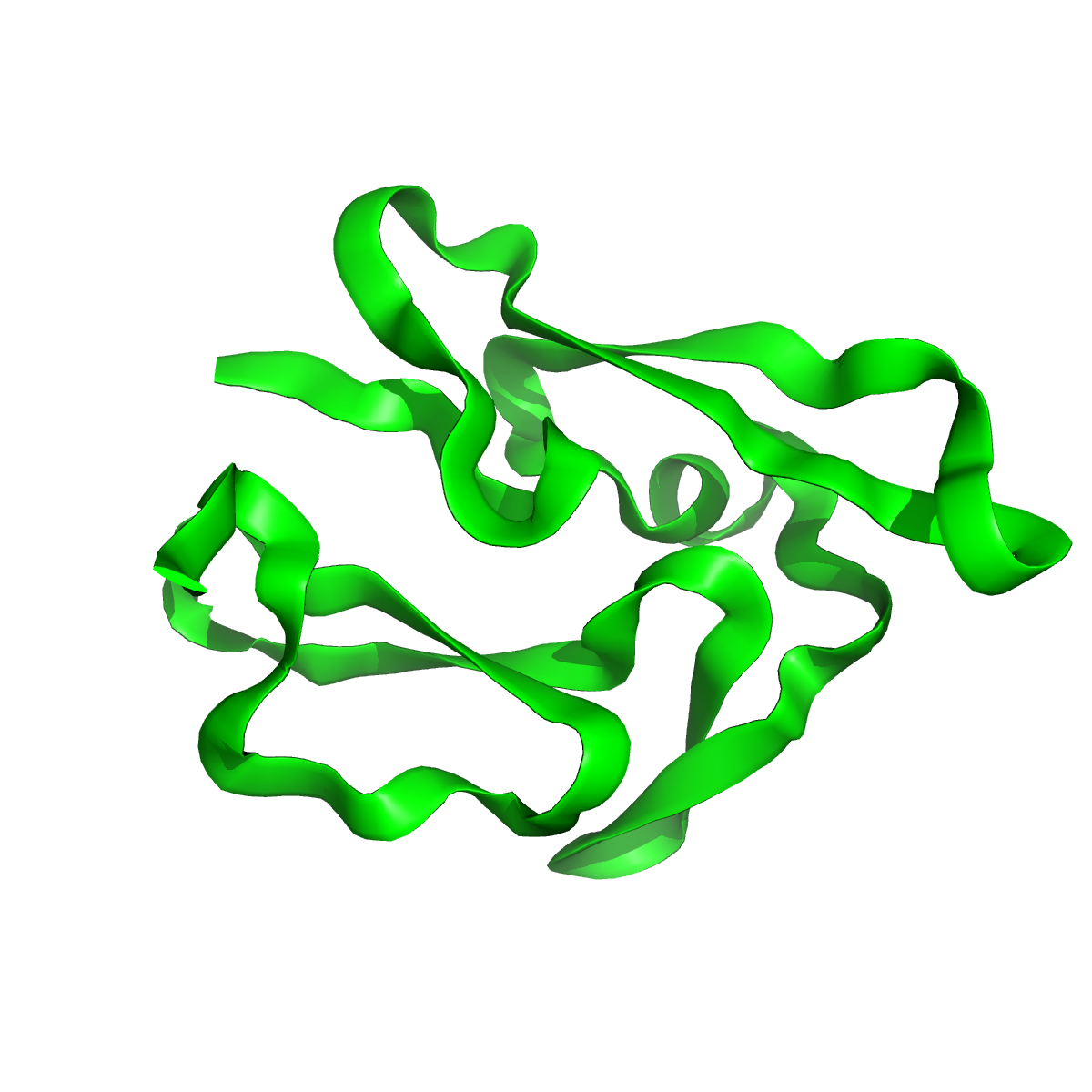}
    \includegraphics[width=0.220\linewidth]{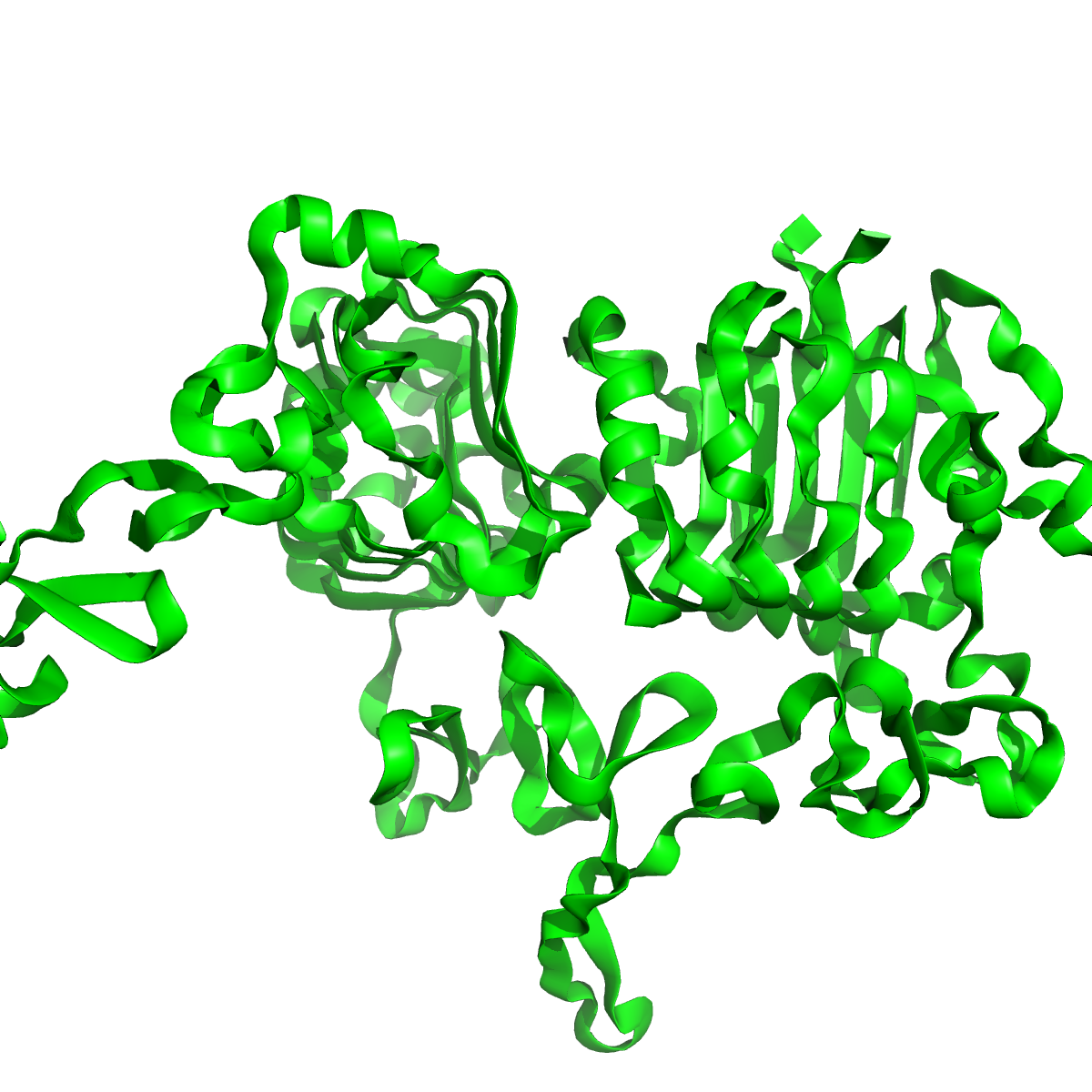}
    \includegraphics[width=0.220\linewidth]{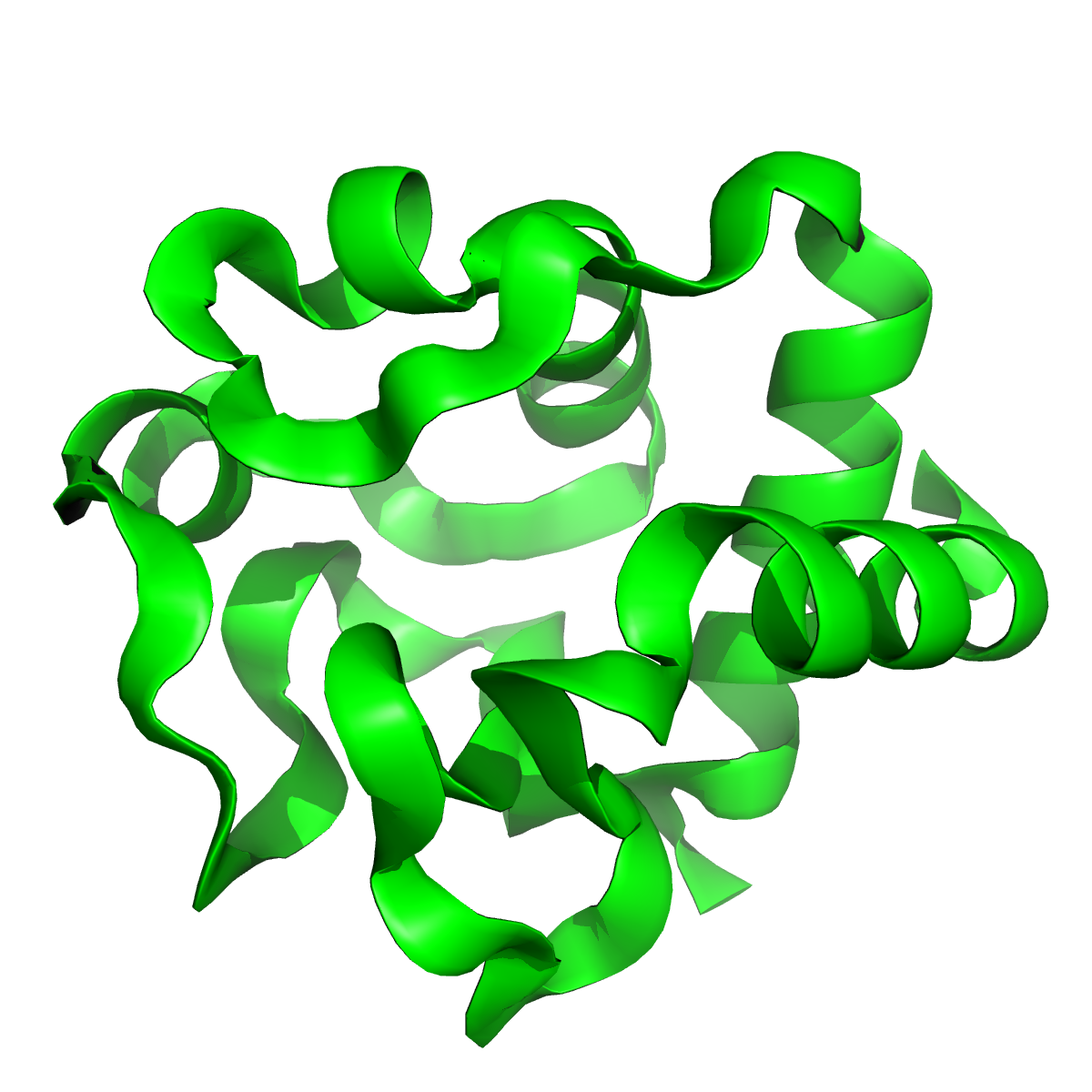}
    \includegraphics[width=0.220\linewidth]{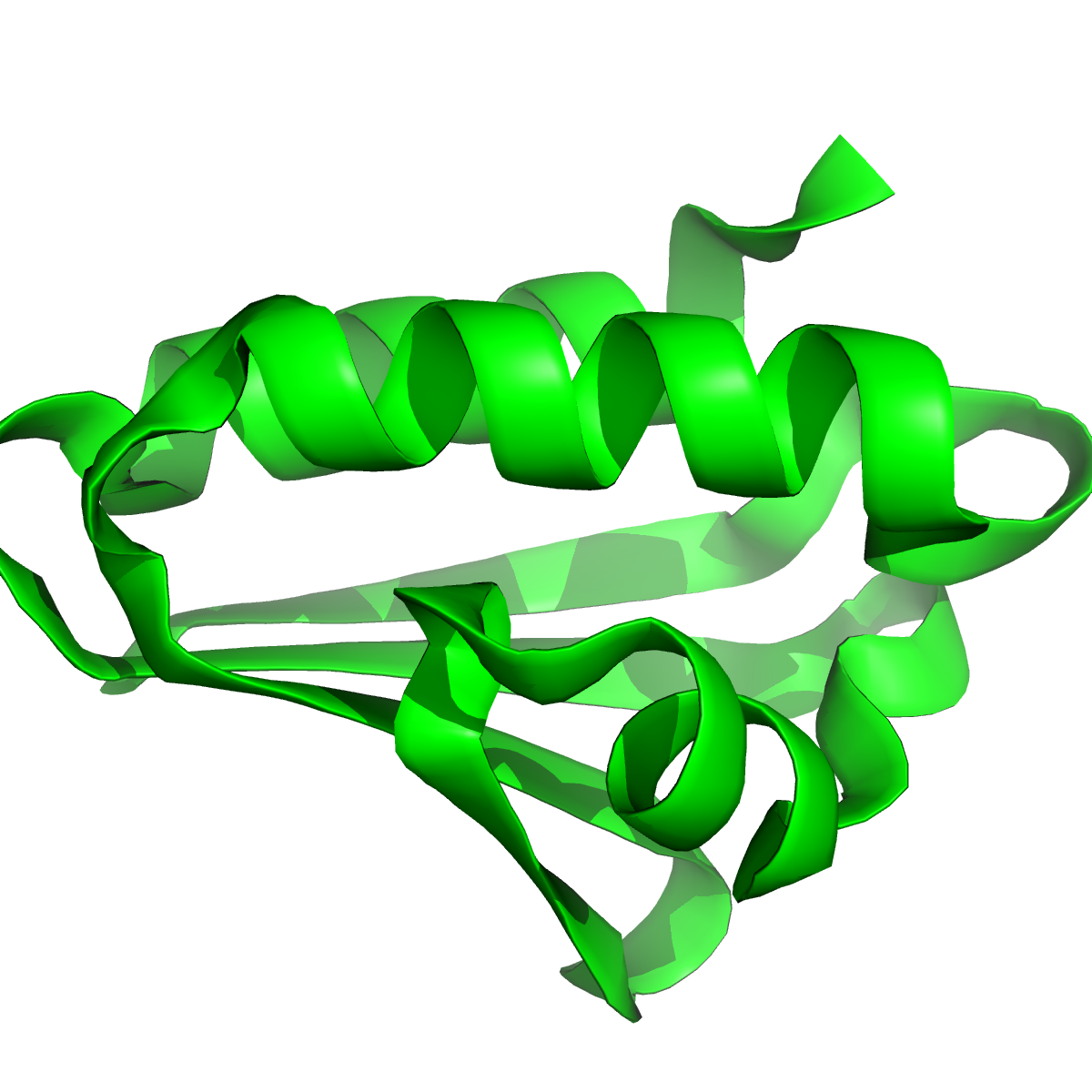}
    \includegraphics[width=0.220\linewidth]{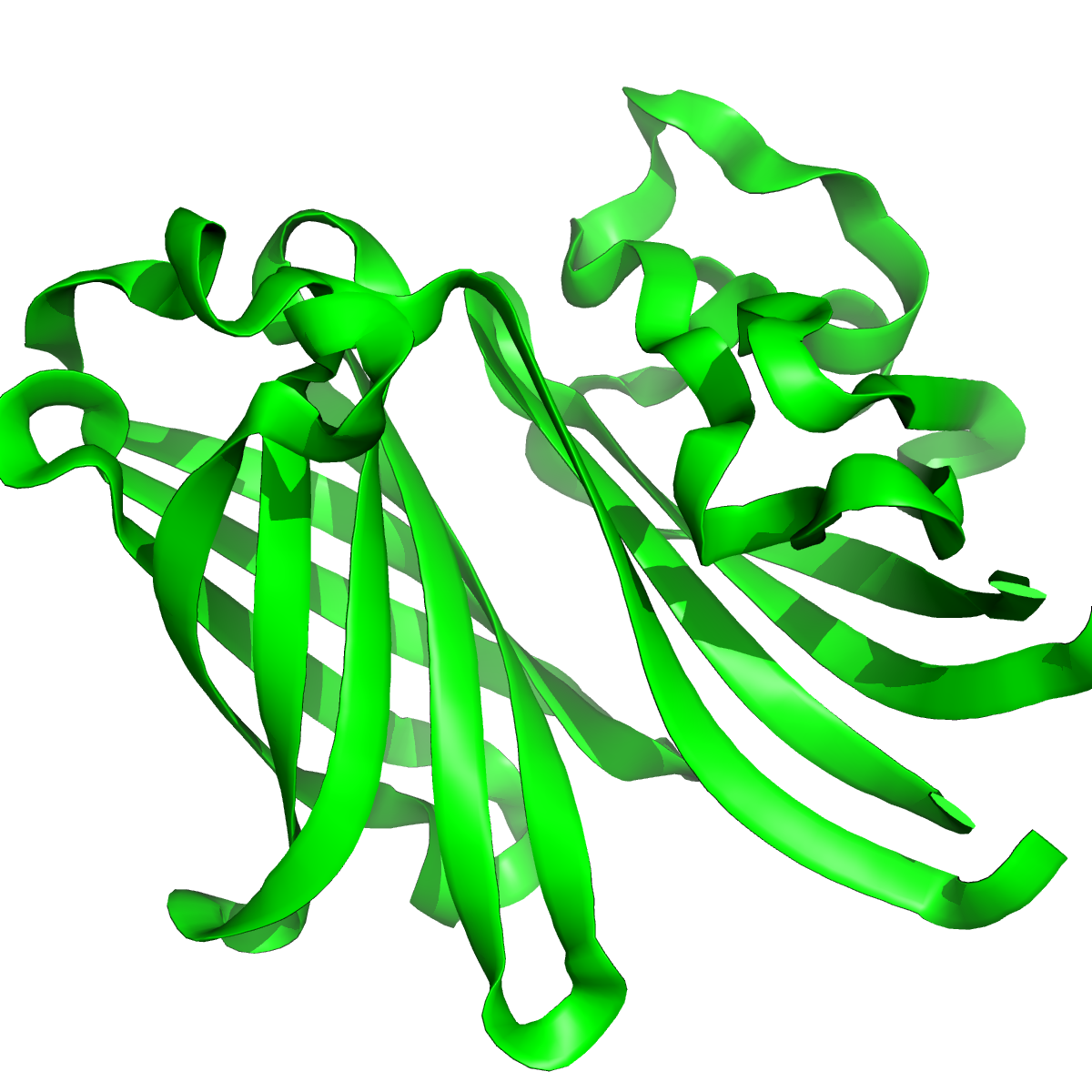}
    \caption{3D structures of the eight target pathogens selected for antibody tasks. From top left to bottom right; 1ADQ\_A, 2DD8\_S, 1NSN\_S, 1OB1\_C, 1S78\_B, 1WEJ\_F, 2JEL\_P , 2YPV\_A.}
    \label{fig:proteintasks}
    % \vspace{-5pt}
\end{figure}
%Importantly 1OB1\_C is a protein found in Malaria and 2DD8\_S a protein found in the SARS-CoV virus.
First, we introduce antibody design and the molecular docking simulator. Next, we explain our experimental setup and baselines, and lastly, we present the analysis and discussion of the results. \looseness=-1

\subsection{Experiment Task and Setup}
\textbf{Antibody design problem:}
Antibodies are large Y-shaped molecules that bind with the antigen at the tip of their variable region~\cite{Chothia1987}. The CDRH3 protein sequence located on the tip of a variable region of the antibody plays a vital role in determining its binding specificity due to its structural diversity~\cite{Chothia1987, xu2000diversity}. Therefore, one goal of the antibody design process can be finding a protein sequence in the CDRH3 region that would lead to an optimal binding site. An example of a binding site between two proteins is illustrated in the right image in Figure~\ref{fig:combtasks}. In our work, we use Absolut!~\cite{robert2021one} to simulate the molecular docking (binding), which computes a lattice view of molecular representation from a sequence and evaluates its binding towards the antigen. We view the Absolut! docking simulator as the objective function, and our goal is to search for a protein sequence, where each character is one of twenty unique Amino Acids ($|\mathcal{X}|=20$), otherwise referred to as AA's, that minimizes the binding energy with a target pathogen. Note, the maximum size of input proteins for docking simulation by Absolut! is $11$, resulting in an extremely large search space of $|\mathbb{S}|= 2.05 \times 10^{14}$. We demonstrate the applicability of RL methods to solve the above combinatorial optimization problem.\looseness=-1

\textbf{Antibody design tasks:} We choose eight highly varied target pathogens to optimise binding energy because of their broad interest in several studies~\cite{robert2021one,akbar2021silico}. The eight antigens are IGG4 Fc Region (1ADQ\_A), SARS-COV Virus Spike glycoprotein (2DD8\_S), SNASE: Staphylococcal nuclease complex (1NSN\_S), MSP1: Merozoite Surface Protein 1 (1OB1\_C), HER2: Receptor protein-tyrosine kinase erbB-2 (1S78\_B), CYC: Cytochrome C (1WEJ\_F), ptsH: Phosphocarrier protein HPr (2JEL\_P) and fHbp: factor H binding protein (2YPV\_A). The first four characters are Protein Data Bank (PDB) followed by a Chain ID. The chain ID identifies a protein responsible for the entering of a pathogen in a host cell. Figure~\ref{fig:proteintasks} shows the 3D structures of the target pathogens.\looseness=-1  
% \subsection{Results and Discussion}
%  \begin{wrapfigure}{R}{0.4\linewidth}
% %  \vspace{-15pt}
%   \centering
%     \includegraphics[\linewidth]{figures/binding/2DD8_S_Struc0_Color_FFCMNILCLRG.png}
%     \caption{SARS-COV antibody task.}
%     \label{fig:examplebi}
% \end{wrapfigure}
%\textbf{Molecular docking simulator:} We use (Absolut~\cite{robert2021one}), a fast and accurate docking simulator for evaluating binding energy of a generated protein with a target pathogen. 
% Molecular docking simulation is \textit{extremely expensive} and \textit{\underline{1000 evaluations can take ~24 hours}}. Despite the heavy computation burden of molecular docking simulation, we are able to run each method for 10 seeds in order to draw confident conclusions.
% \paragraph{MaxSAT} Weighted MaxSAT problem has the objective of determining the maximum number of clauses of a given mathematical expression, that can be made true by an assignment of true/false to a sequence of variables. We look at three MaxSAT problems involving 28, 43 and 60 binary variables with structure space size of as large as $|\mathbb{S}|=1.15 \times 10^{18}$. See Appendix for further details~\ref{app:maxcut}.

% \paragraph{MaxSMT} We investigate a MaxSMT combinatorial optimization problem which contains $10$ combinatorial variables which can each take $20$ integer values, with a total state space size of $|\mathbb{S}|= 1.02 \times 10^{13}$. Given a mathematical expression $\sum_{i=0}^{9} (x_i-10)=0$, is there any integer variable assignment that satisfies this formula?. We call this MaxSMT-10-20.

\textbf{SQL and Baseline Settings} We evaluate two sequential structured policy evaluation strategies with SQL using beam search with k=1 (Greedy) and k=20 (Beam). We also evaluate a non-sequential strategy (Masked). For structured policy improvement we use $\mathcal{S}$-greedy, chosen from experimental results (see ablation results \ref{subsec:results}). We compare to two  variations (Critic \&  MaxB) of a SOTA RL algorithm for combinatorial optimization (with structural priors), adapted to our setting from~\cite{BelloPLNB16, kool2018attention}, which we named Structured Policy Gradients (SPG). See Appendix~\ref{sec:spg} for details of SPG.  Baselines include popular combinatorial optimization algorithm simulated annealing~\cite{bertsimas1993simulated} (SA), random search (RS), and the unstrustured counterparts of SQL and PG: Q-learning (QL), and policy gradients (PG).\looseness=-1

\textbf{Implementation details} 
% All experiments were completed on multiple servers with with 1TB of RAM and six Intel Xeon E5 CPUs each, with each experiment requiring only one Intel Xeon E5 CPU.
In order to ensure we only compare effects from structural priors, all methods use the exact same transformer architecture trained with batch size 32, Adam~\cite{kingma2014adam} optimiser with learning rate 0.001. The only differing factor
%in QL, PG, SPG (Critic), SPG (MaxB), SQL (Greedy), SQL (Masked), SQL (Beam)
between methods is the manner in which the Transformer is utilised (see Appendix~\ref{app:arcs} for exact details of neural architectures). Each trial completed on one Intel Xeon E5 CPU with 200GB RAM. \looseness=-1
% {\color{red}Please note, we intentionally do not compare to other variations of QL and PG as we want to also compare how the structural prior effects algorithm performance versus the unstructured priors. ??}\looseness=-1

\begin{figure}[th!]
\includegraphics[width=1\linewidth]{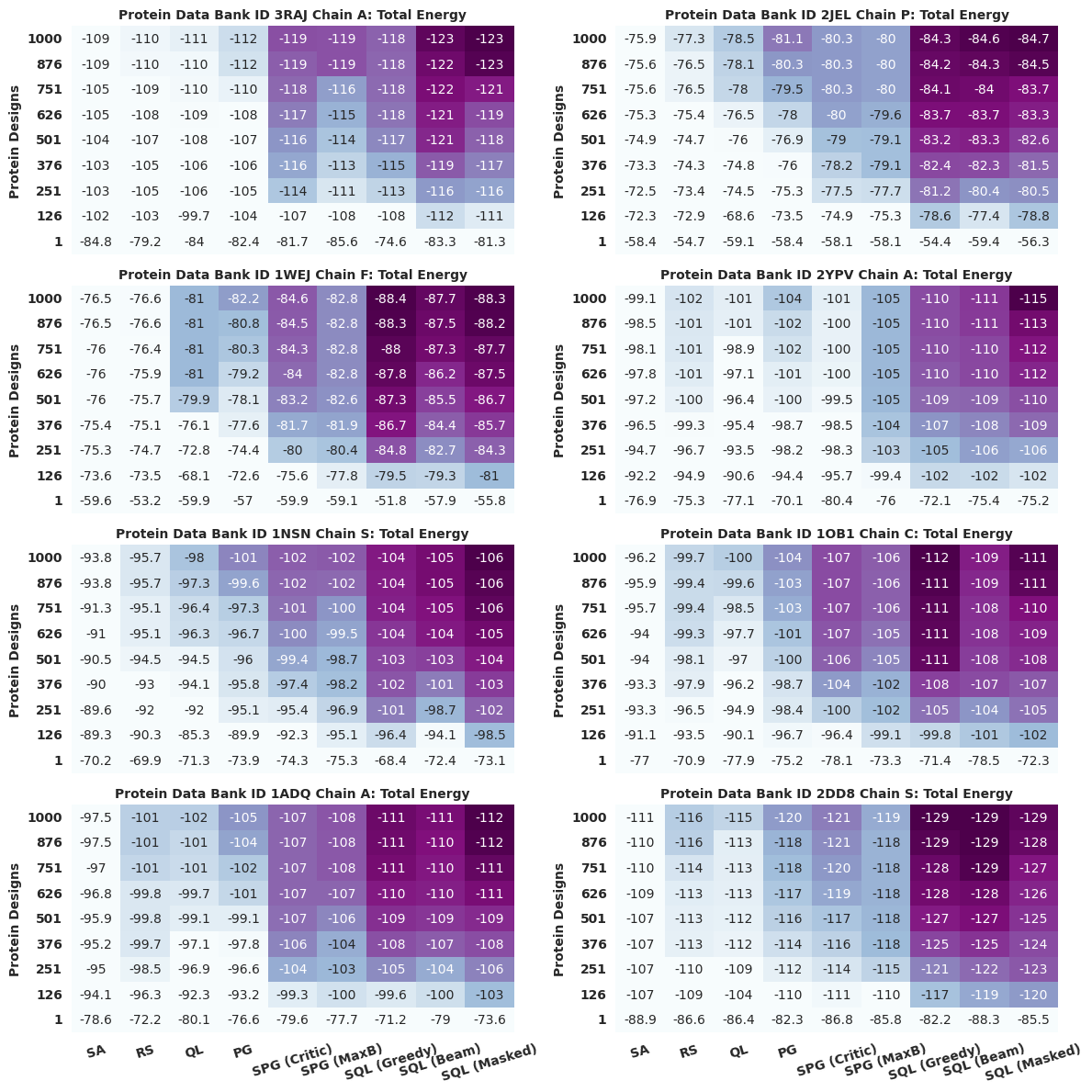}
% \caption{Protein optimization Tasks. We show a summary table of algorithm performance across eight protein optimization tasks. where each block is a protein interaction task with a target protein defined by the PDB and Chan ID on the y label. Each column is a method, with each cell being the total energy, averaged over 10 seeds, for every 124 protein designs. With the y ticks showing the number of protein designs. We colour each cell based on how different it is to Simulated Annealing final average performance, with darker colours increased improvement. See Appendix~\ref{app:breakdown} for breakdown per task}
% \caption{Protein optimization Tasks. We show a heatmap where each block is a protein interaction task with a target protein defined by the PDB and Chan ID on the y label. Each column is a method, with each cell being the total energy, averaged over 10 seeds, for every 124 protein designs. With the y ticks showing the number of protein designs. We colour each cell based on how different it is to Simulated Annealing final average performance, with darker colours increased improvement.}
% \vspace{-20pt}
\caption{Binding energy performance of baselines and our methods for eight antibody design tasks, averaged across $10$ random trials per task. Lower (darker shade) energy is better. %SQL consistently designs high affinity sequences.}
}
\label{fig:absoluttable}
\end{figure}

\subsection{Results and Discussion}
\label{subsec:results}

%Here we show the results of eight antibody design task and present discussion of our key findings. Table~\ref{fig:absoluttable} and 
Figure~\ref{fig:absoluttable} shows heatmap results of the eight antibody design tasks. Each block is a unique antibody task (see Figure~\ref{fig:proteintasks}). Each column is a method, with each cell being the binding energy averaged over $10$ seeds, while the y-axis shows the number of protein designs over time. We colour each cell based on how much it improves upon Simulated Annealing's final average performance, with darker colours highlighting increased improvement.\looseness=-1

We see SQL variants being the top performing amongst off- and on-policy agents. Generally, Q-learning performs worse than policy gradients however, when comparing their combinatorial counterparts (SQL vs SPG), we see this trend reversed. Generally the greedy baseline helps SPG versus using a critic, consistent with~\cite{kool2018attention}. Increasing the beam size from k=1 (Greedy) to k=20 (Beam) for SQL improves performance slightly however, the highest performing variant is the only non-sequential method evaluated -- SQL (Masked). We believe this shows the promise of non-sequential generation of proteins, where the model learns to directly generate the best performing structures.\looseness=-1

\begin{table}[t!]
\setlength\tabcolsep{1.85pt}
\small
\centering
\scalebox{0.9}{\begin{tabular}{l|ccccccc}
\toprule
\diagbox[innerwidth=3.0cm,font=\footnotesize]{Quantiles}{\textbf{Antigens}}                & \textbf{1ADQ\_A} & \textbf{1NSN\_S} & \textbf{1OB1\_C} & \textbf{1WEJ\_F} & \textbf{2YPV\_A} & \textbf{3RAJ\_A} & \textbf{2JEL\_P} \\ \hline
\textbf{> All 6.9M}   & -108.53          & -107.16          & -108.78          & -86.03           & -114.47          & -116.74          & -86.05           \\
\textbf{>0.01\%} & -102.62          & -99.09           & -101.56          & -79.15           & -103.86          & -107.95          & -79.53           \\
\textbf{>0.1\%}  & -98.51           & -94.85           & -97.45           & -76.56           & -99.92           & -104.5           & -75.82           \\
\textbf{>1\%}    & -94.19           & -89.99           & -92.52           & -73.24           & -95.18           & -100.3           & -71.41           \\
\textbf{>5\%}    & -90.03           & -85.46           & -88.07           & -70.41           & -91.16           & -96.49           & -67.84           \\
\textbf{>95\%}   & -52.71           & -46.64           & -46.64           & -37.21           & -53.62           & -53.97           & -38.63          \\
\bottomrule
\end{tabular}}
\caption{Energy quantiles of 6.9M experimentally obtained antibody sequences (all sub-sequences of length 11) available from Absolut database~\cite{robert2021one}. 2DD8\_S not reported due to missing data.}
\label{table:thresholds}
\end{table}

\begin{figure}[th!]
    \centering
    \includegraphics[width=\linewidth]{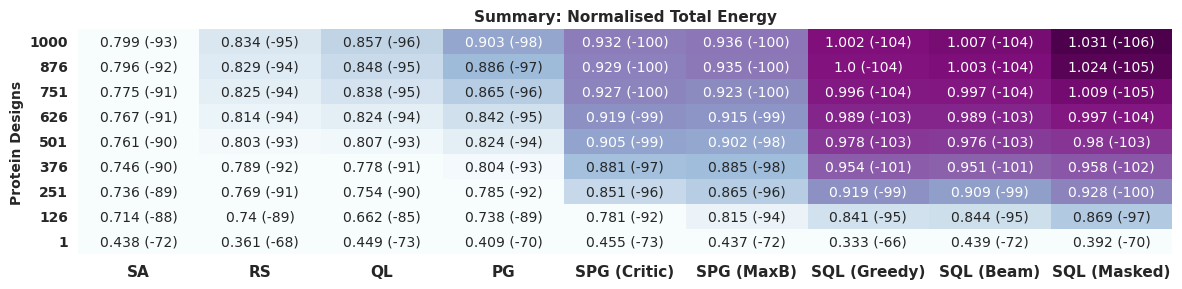}
    \caption{Heat map of normalised (using Table~\ref{table:thresholds}) binding energy averaged across antigen tasks and random seeds. We observe SQL (Masked) ranks as the best method and only SQL variants, on average, achieve better energy scores ($>1.0$) than all 6.9M biologically obtained proteins.}
    \label{fig:absolutsummary}
    % \vspace{-10pt}
\end{figure}

\begin{wrapfigure}{R}{0.4\linewidth}
  \centering
  \vspace*{-3ex}
    \includegraphics[width=\linewidth]{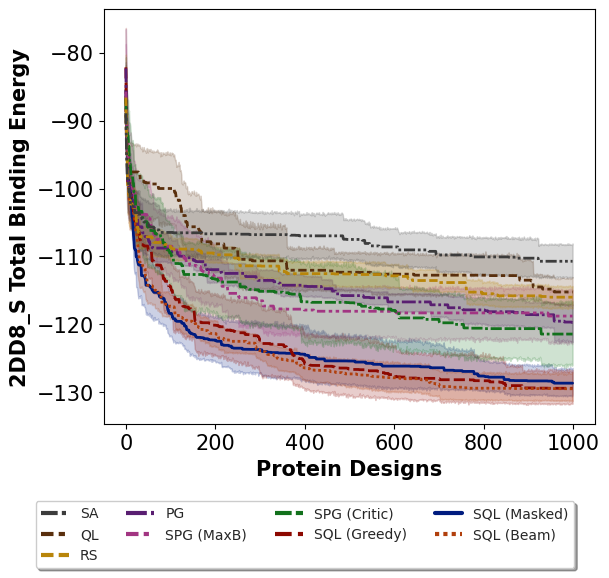}
    \caption{SARS-COV antibody task.}
    \label{fig:sarscov}
    \vspace*{-3ex}
\end{wrapfigure}

%Next, we study where different methods stand in comparison to experimentally known database~\cite{robert2021one}.
We are also interested in how different methods perform with respect to a known database of solutions~\cite{robert2021one}. \cite{akbar2021silico} used quantiles of energy distribution from a known database of 6.9 million (6.9M) biologically obtained proteins to categorise them as low, high, very high, and super binders (top 5/1/0.1/0.01\%), based on their binding towards an antigen. This allows for analysing how good the ML methods are in designing sequences over known binding categories. The thresholds of the categories for seven~\footnote{Not all eight tasks due to missing data for 2DD8\_S.} antigens are reported in Table~\ref{table:thresholds}. In combination with Figure~\ref{fig:absoluttable}, we can see that SQL reaches high to very high affinity in much fewer evaluations compared to baselines. Moreover, it consistently outperforms the best known sequence available in the database within $1000$ evaluations.\looseness=-1

Figure~\ref{fig:absolutsummary} further displays the average normalised energy across all random trials and seven antigen tasks~\footnote{Except 2DD8\_S due to the missing quantiles needed for normalisation.}. We normalise using the tail ends (best and worse) of the energy distribution acquired from~\citep{robert2021one}. A normalised score of $>1.0$ means that on average, the method finds antibody sequences with better energy scores than all 6.9M biologically obtained protein sequences. Importantly, \textit{only SQL variants} are able to reach the tremendously difficult average normalzied score of $\geq1.0$. In all cases, we see that structural priors improve performance vs the unstructured equivalent models (QL vs SQL) and (PG vs SPG).\looseness=-1

As an example to illustrate the efficacy of SQL for the challenging optimization of antibody designs, Figure~\ref{fig:sarscov} shows convergence plots of the tested methods on SARS-COV. We can see that while SQL variants obtain the best energy scores, they do so both fastest and with the lowest standard deviation. This observation can be generalised across all tasks (Appendix Figures~\ref{fig:table_var} and~\ref{fig:pdb_all}).\looseness=-1

%\paragraph{Substructure Analysis}
Analysing the results from all seeds and methods on chain A of 1ADQ, we discovered over $300$ unique protein sequences reaching the best binding energy score $(-112.59)$. This highlights the complexity of the problem and how so many related local minima exist. Importantly, $>99\%$ (all bar three) top performing proteins were found by SQL,
%, with two sequences from SPG and one from QL. Given that SQL finds $>99\%$ of the best proteins, this proves
showing that SQL is able to consistently find diverse and well performing antibodies when optimising for binding energy over a molecular simulator. See Appendix Figure~\ref{fig:msa} for additional analysis.\looseness=-1

%TODO FINISH CLIPPING
\begin{figure}[th!]
    \centering
    \includegraphics[width=0.4\linewidth,trim={0.8cm 0.4cm 1.9cm 1.4cm},clip]{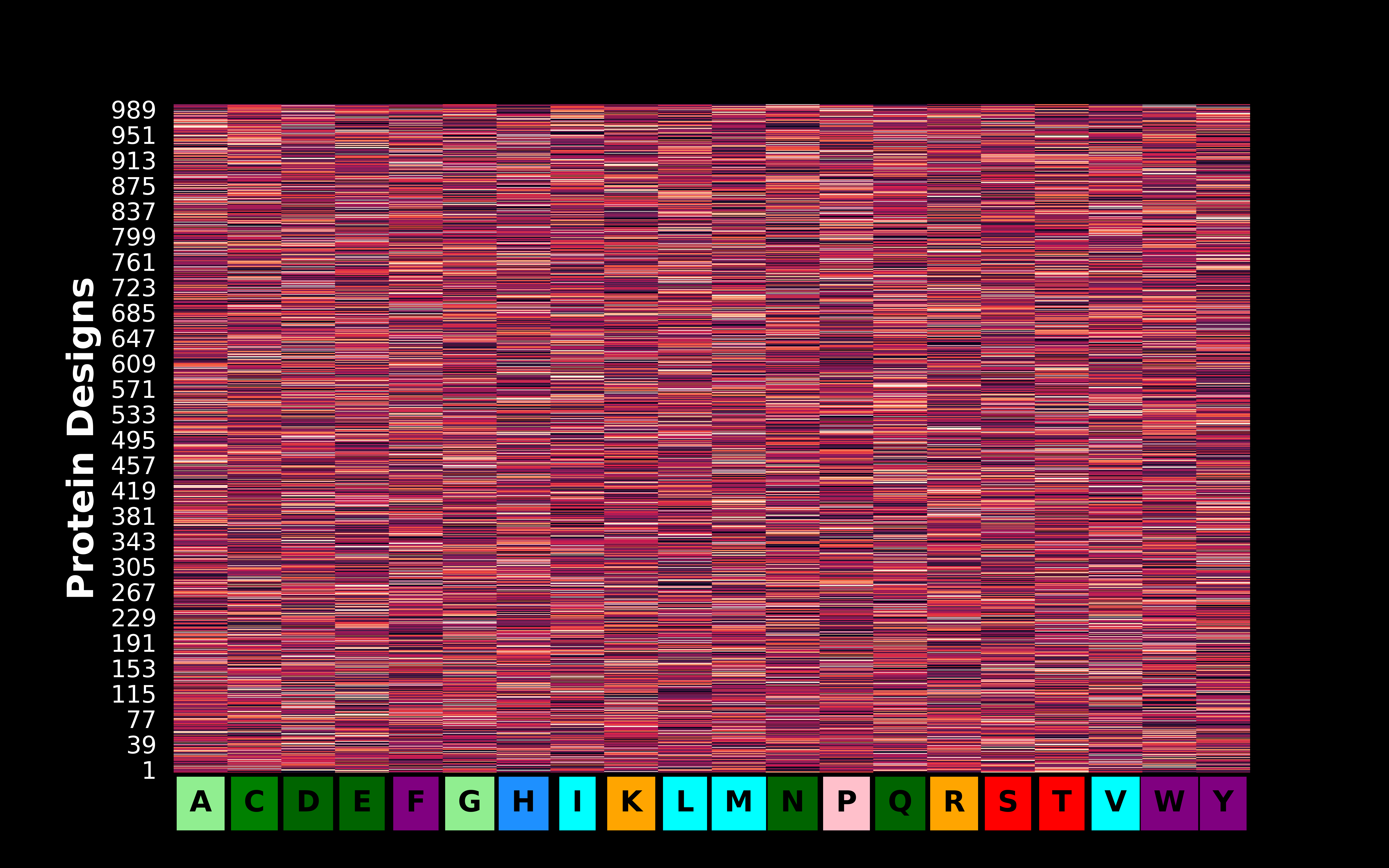}
    \includegraphics[width=0.4\linewidth,trim={0.8cm 0.4cm 1.9cm 1.4cm},clip]{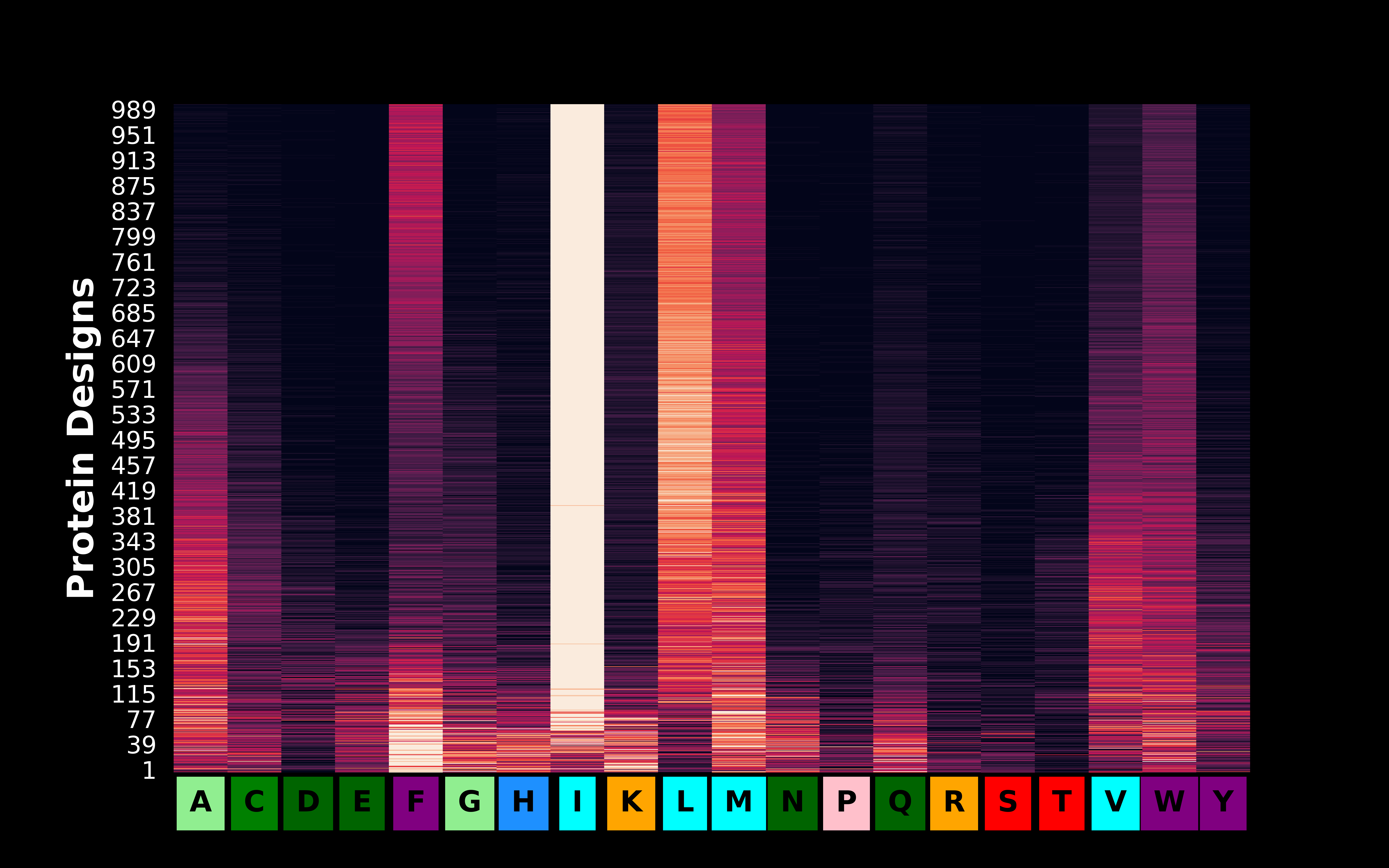}
     \includegraphics[width=0.4\linewidth,trim={0.8cm 0.4cm 1.9cm 1.4cm},clip]{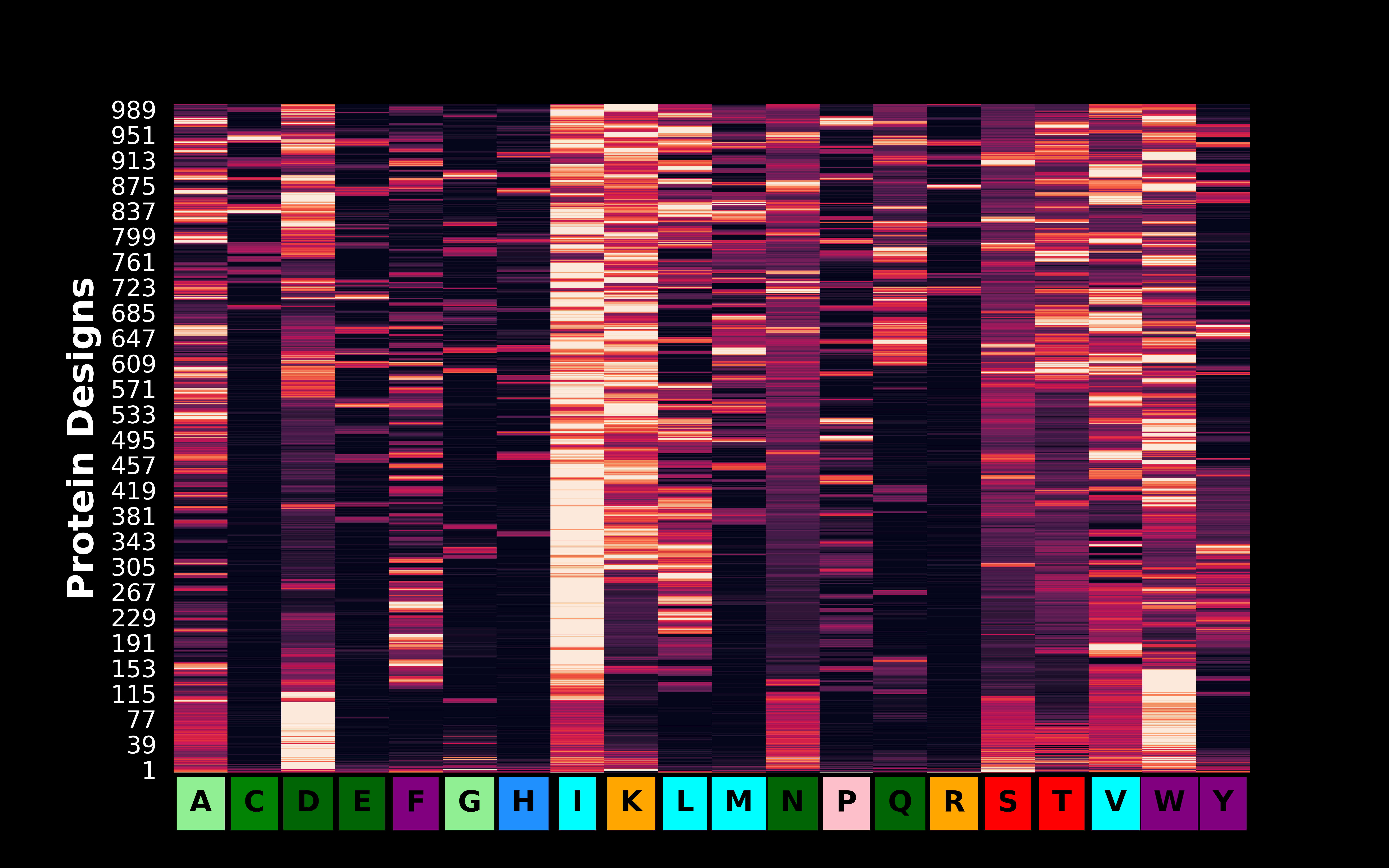}
    \includegraphics[width=0.4\linewidth,trim={0.8cm 0.4cm 1.9cm 1.4cm},clip]{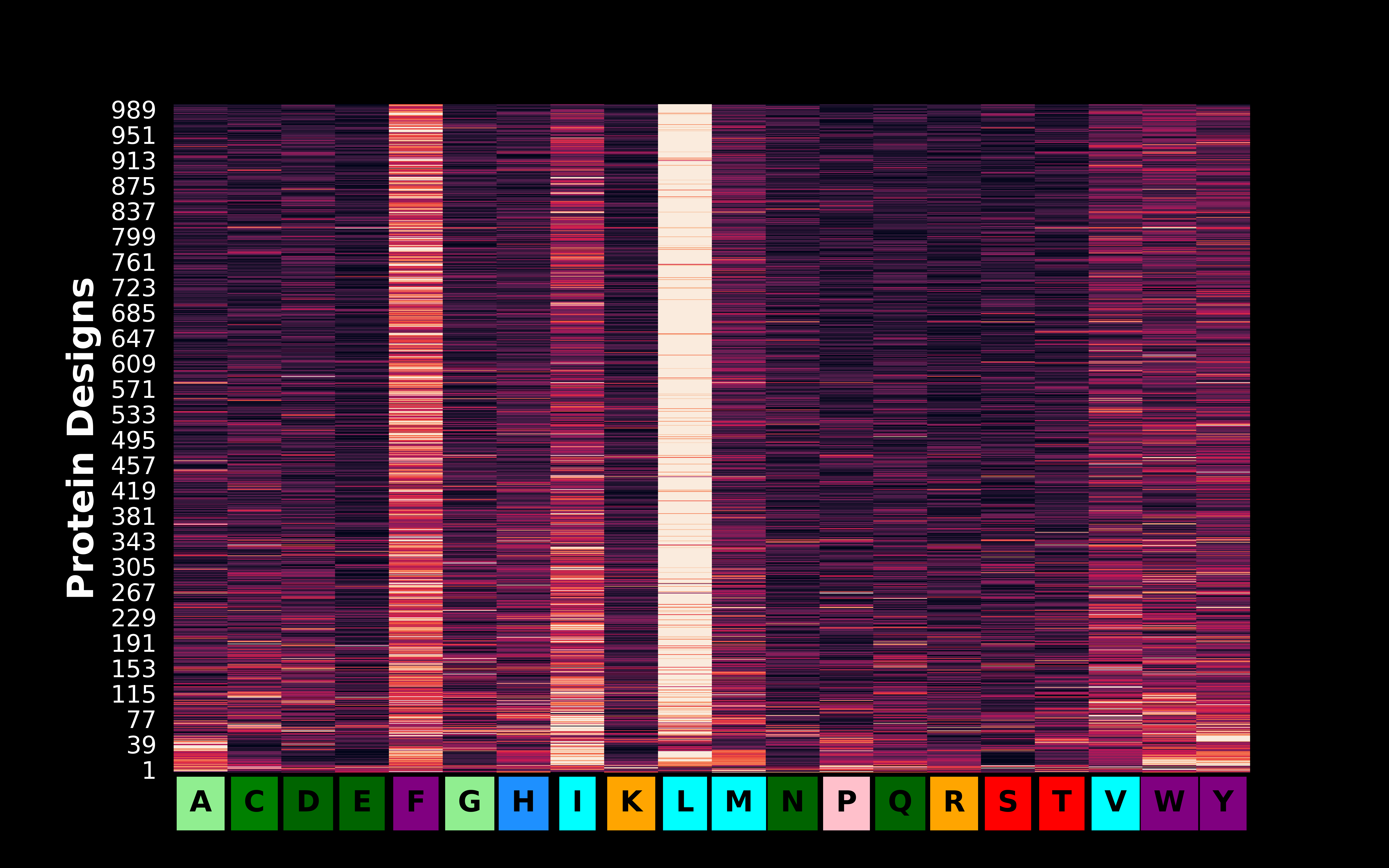}
    \caption{Heatmap of proportions of AA's selected for SARS-COV (2DD8\_S), across 10 seeds per step. The brighter the colour, the higher the proportion. From top left to bottom right: RS, SPG (Critic), QL, SQL. Importantly, we see SQL both focusing on high energy AA's, as demonstrated in Figure~\ref{fig:sarscov}, whilst continuing diverse exploration, as shown by bright shading almost everywhere.}
    \label{fig:diveristy}
    %   \vspace{-10pt}
\end{figure}

\begin{wrapfigure}{R}{0.4\linewidth}
  \centering
  \vspace*{-2ex}
    \includegraphics[width=\linewidth,trim={0 0 0 8.2cm},clip]{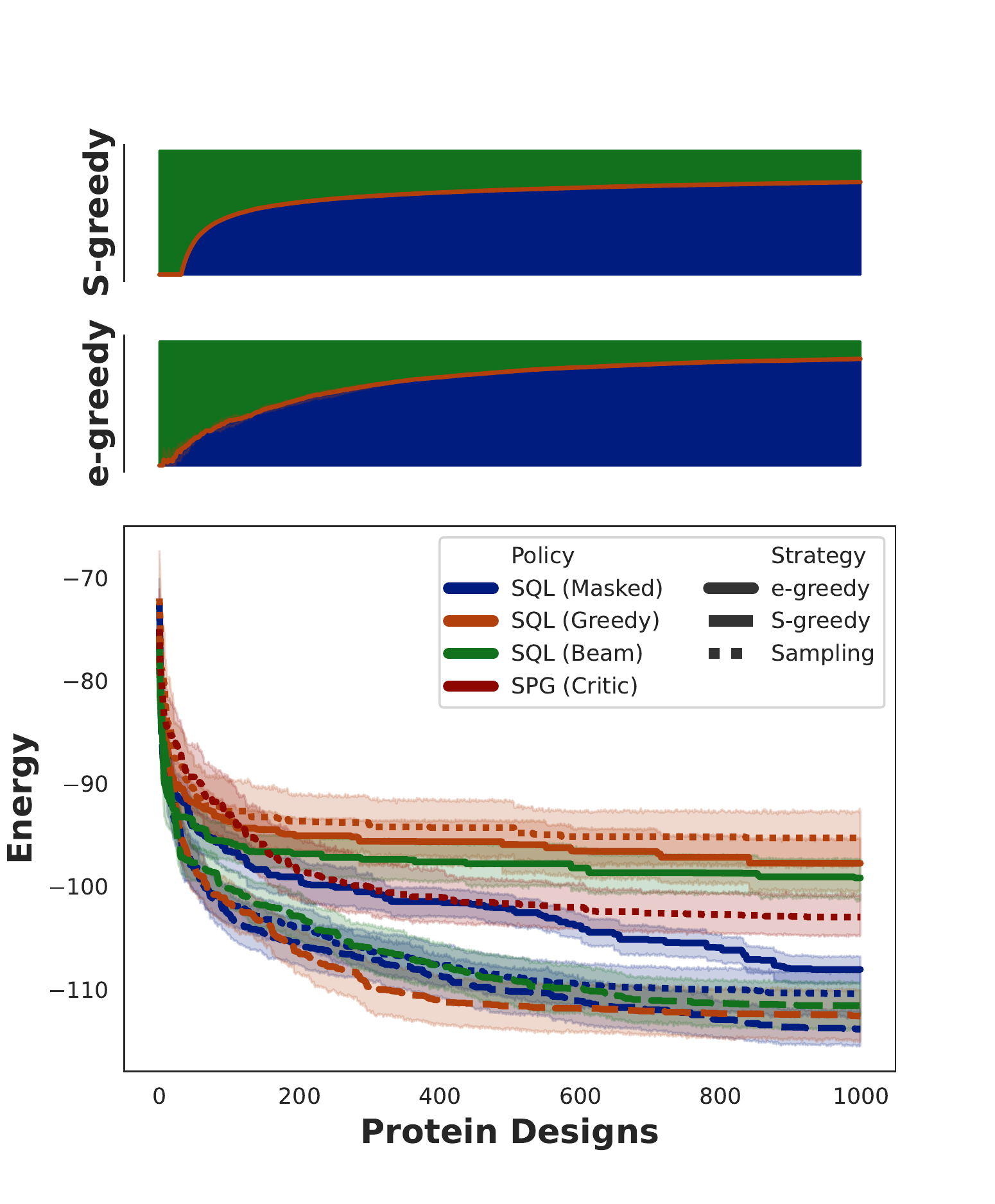}
    \caption{Structural priors ablation.}
    \label{fig:ablation}
    \vspace*{-1ex}
\end{wrapfigure}

This diversity in solutions is further illustrated in Figure~\ref{fig:diveristy}, which traces the amino acid selection across the optimization process for differing agents on SARS-COV over all seeds. For each amino acid (x-axis), we show the proportion it is selected at each structure suggestion step (y-axis). Both SQL and QL show diverse traces across the task; while QL seems to continue to explore, SQL eventually focuses on exploiting variations of a few core amino acids that provide excellent energy scores. SPG seems to exploit early on, which could explain the under-performance against SQL, which even during the exploitation phase is constantly exploring other solutions.\looseness=-1

%\paragraph{Structured Policy Improvement \& Evaluation Analysis}\label{res:spi}
In the final experiment, we show the influence of varying the policy evaluation and improvement strategies. For each evaluation strategy (greedy, beam search, masked) we run every possible improvement strategy ($\epsilon$-greedy, $\mathcal{S}$-greedy, sampling\footnote{For Beam we do not include sampling exploration as it can not be implemented together.}). We compare to SPG (Critic), the overall next best performing non-SQL method in previous experiments. We run this ablation study for 20 seeds for each method (Figure~\ref{fig:ablation}). It seems the best improvement strategy for all evaluation strategies is $\mathcal{S}$-greedy, while performance decreases with $\epsilon$-greedy and sampling. Interestingly for masked policy evaluation, sampling seems to quite significantly outperform $\epsilon$-greedy.\looseness=-1

% \begin{figure}
%     \centering
%     % \includegraphics[width=0.4\linewidth]{figures/synthetic/Synthetic_Task_synthetic10_legend_True.png}
%     \includegraphics[width=0.5\linewidth,trim={0 0 0 8.2cm},clip]{figures/absolut/Ablation.pdf}
%     \caption{Ablation of various methods for Structural Policy Evaluation (Section~\ref{sec:spe}) and Structural Policy Improvement (Section~\ref{sec:spi}) compared to SPG (Critic). We see a clear consensus of S-greedy structured policy improvement performing the best with all structured policy evaluation strategies. However, sampling with SQL (Masked) improves upon e-epsilon, however the opposite is true for SQL (AR k=1).}
%     \label{fig:ablation}
%     % \caption{Easy task of finding a protein which has only one amino acid repeated, where this is randomly selected at the start of each task. In this case we expect the agent to exploit fairly quickly due to the simple correlation. We indeed see the criterion between critics learns to favour accepting agents exploitative actions rather than rejecting these.}
%     \label{fig:synthetic10}
% \end{figure}

\section{Conclusion and Future Work}
We have introduced Structured Q-learning, an extension to classic Q-learning with structural priors. We have demonstrated the effectiveness of SQL on the combinatorial domain of antigen construction. Using a molecular docking simulator we evaluated SQL on optimising protein sequences to bind to various target pathogens, where we observed it significantly improves upon existing RL agents. Importantly, all learning algorithms use the same neural architecture, differing only in how this architecture is utilised. In the future we would like to extend SQL to other combinatorial domains, as well as similarly adapt other off-policy Reinforcement Learning methods for combinatorial optimization.\looseness=-1

\newpage

\bibliographystyle{ieeetr}
\bibliography{mlprotein}

\section*{Checklist}

% %%% BEGIN INSTRUCTIONS %%%
% The checklist follows the references.  Please
% read the checklist guidelines carefully for information on how to answer these
% questions.  For each question, change the default \answerTODO{} to \answerYes{},
% \answerNo{}, or \answerNA{}.  You are strongly encouraged to include a {\bf
% justification to your answer}, either by referencing the appropriate section of
% your paper or providing a brief inline description.  For example:
% \begin{itemize}
%   \item Did you include the license to the code and datasets? \answerYes{}
%   \item Did you include the license to the code and datasets? \answerNo{The code and the data are proprietary.}
%   \item Did you include the license to the code and datasets? \answerNA{}
% \end{itemize}
% Please do not modify the questions and only use the provided macros for your
% answers.  Note that the Checklist section does not count towards the page
% limit.  In your paper, please delete this instructions block and only keep the
% Checklist section heading above along with the questions/answers below.
% %%% END INSTRUCTIONS %%%

\begin{enumerate}

\item For all authors...
\begin{enumerate}
  \item Do the main claims made in the abstract and introduction accurately reflect the paper's contributions and scope?
    \answerYes{}
  \item Did you describe the limitations of your work?
    \answerYes{We discuss limitation of method and experiments in experimental section.}
  \item Did you discuss any potential negative societal impacts of your work?
    \answerYes{}
  \item Have you read the ethics review guidelines and ensured that your paper conforms to them?
    \answerYes{}
\end{enumerate}

\item If you are including theoretical results...
\begin{enumerate}
  \item Did you state the full set of assumptions of all theoretical results?
    \answerYes{}
        \item Did you include complete proofs of all theoretical results?
    \answerYes{}
\end{enumerate}

\item If you ran experiments...
\begin{enumerate}
  \item Did you include the code, data, and instructions needed to reproduce the main experimental results (either in the supplemental material or as a URL)?
    \answerYes{Code and a script included.}
  \item Did you specify all the training details (e.g., data splits, hyperparameters, how they were chosen)?
    \answerYes{In experiment section and appendix.}
        \item Did you report error bars (e.g., with respect to the random seed after running experiments multiple times)?
    \answerYes{}
        \item Did you include the total amount of compute and the type of resources used (e.g., type of GPUs, internal cluster, or cloud provider)?
    \answerYes{}
\end{enumerate}

\item If you are using existing assets (e.g., code, data, models) or curating/releasing new assets...
\begin{enumerate}
  \item If your work uses existing assets, did you cite the creators?
    \answerYes{}
  \item Did you mention the license of the assets?
    \answerNA{}
  \item Did you include any new assets either in the supplemental material or as a URL?
    \answerYes{}
  \item Did you discuss whether and how consent was obtained from people whose data you're using/curating?
    \answerNA{}
  \item Did you discuss whether the data you are using/curating contains personally identifiable information or offensive content?
    \answerNA{}
\end{enumerate}

\item If you used crowdsourcing or conducted research with human subjects...
\begin{enumerate}
  \item Did you include the full text of instructions given to participants and screenshots, if applicable?
    \answerNA{}
  \item Did you describe any potential participant risks, with links to Institutional Review Board (IRB) approvals, if applicable?
    \answerNA{}
  \item Did you include the estimated hourly wage paid to participants and the total amount spent on participant compensation?
    \answerNA{}
\end{enumerate}

\end{enumerate}

\setcounter{equation}{0}
\setcounter{thm}{0}
\setcounter{defn}{0}
\setcounter{figure}{0}
\setcounter{table}{0}
\setcounter{section}{0}

\renewcommand{\thethm}{A\arabic{thm}}
\renewcommand{\thedefn}{A\arabic{defn}}
\renewcommand{\theprop}{A\arabic{prop}}
\renewcommand{\thelem}{A\arabic{lem}}
\renewcommand{\thesection}{A\arabic{section}}

\renewcommand{\theequation}{A\arabic{equation}}
\renewcommand{\thetable}{A\arabic{table}}
\renewcommand{\thefigure}{A\arabic{figure}}
\renewcommand\thesubfigure{(\alph{subfigure})}

\renewcommand\ptctitle{}
\clearpage
\appendix 
\renewcommand{\partname}{Appendices}
\renewcommand{\thepart}{}
\doparttoc
\part{} 
\parttoc
\section{Ethical Considerations \& Potential Negative Societal Impacts}
It is well known that the methods developed for molecular optimization of therapeutic targets, also have the dual use to optimise for toxic targets~\cite{urbina2022dual}. Although domain knowledge is still required in chemistry/ toxicology in order to develop harmful molecules, we must continue to raise awareness of the dual use in order to further engage the community in responsible science.

In Appendix~\ref{app:breakdown} we present additional experiment results. In Appendix~\ref{app:sql} we present the full algorithm for SQL alongside additional details. In Appendix~\ref{app:VAMP} we present further details of the variable allocation markov decision process (VAMP). Lastly, in Appendix~\ref{sec:spg} we will cover the implementation details of the  Structured Policy Gradients baselines.

\section{Additional Experimentation}\label{app:breakdown}

\begin{figure}[h!]
    \centering
    \includegraphics[width=0.3\linewidth]{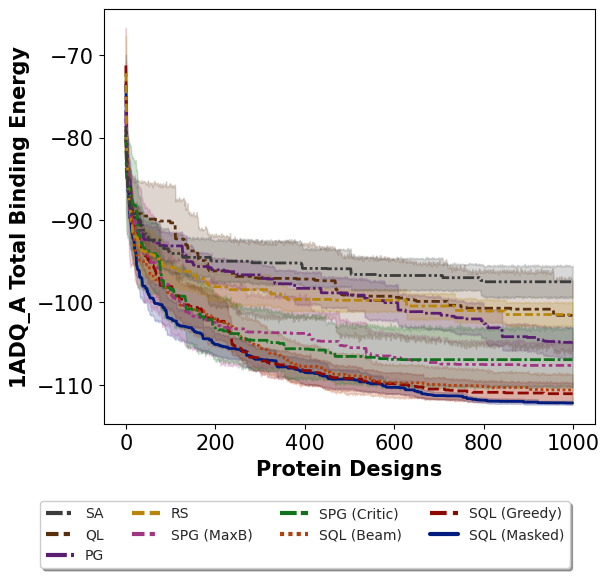}
    \includegraphics[width=0.3\linewidth]{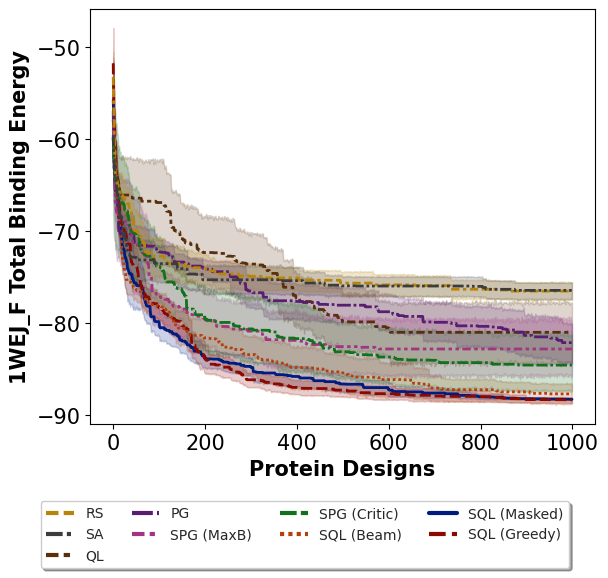}
    \includegraphics[width=0.3\linewidth]{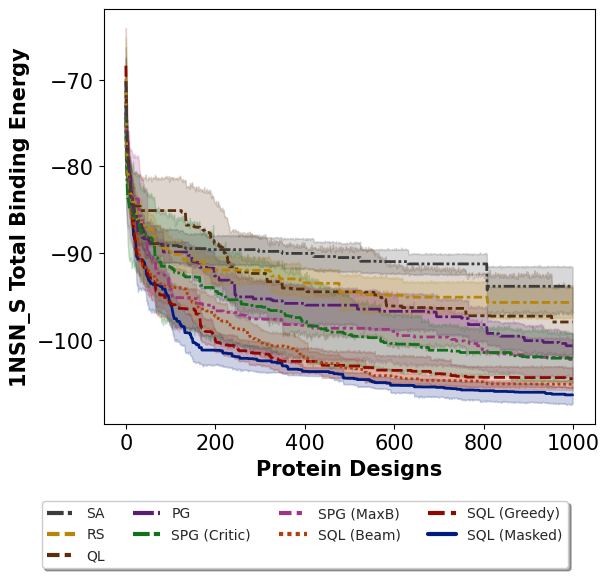}
    \includegraphics[width=0.3\linewidth]{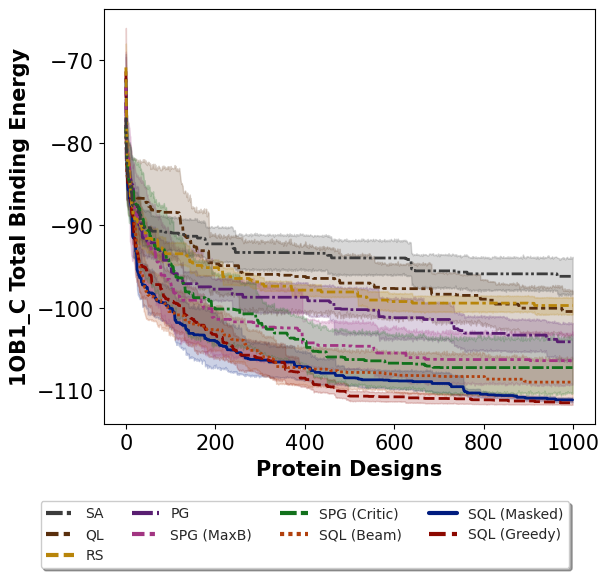}
    \includegraphics[width=0.3\linewidth]{figures/general/PDB_2DD8_Chain_S_core_legend_True.png}
    \includegraphics[width=0.3\linewidth]{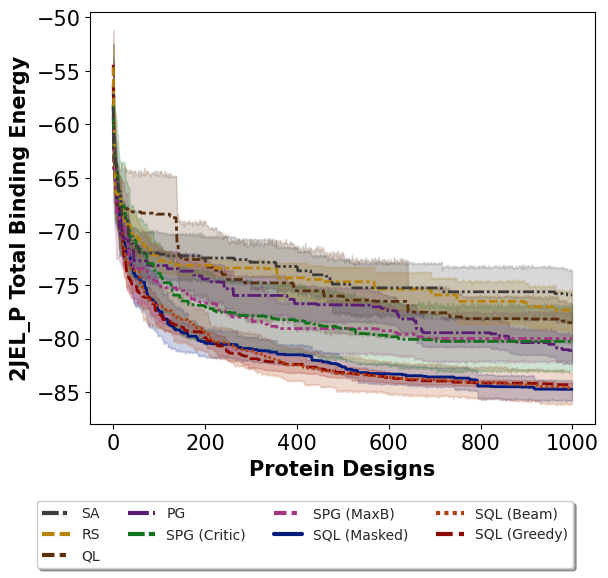}
        \includegraphics[width=0.3\linewidth]{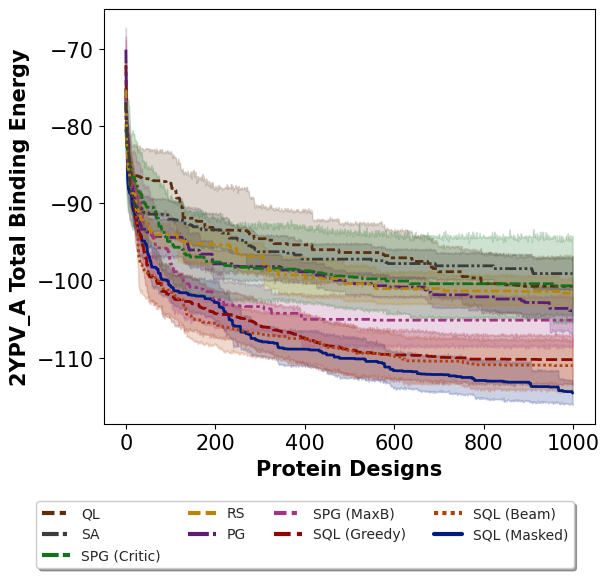}
    \includegraphics[width=0.3\linewidth]{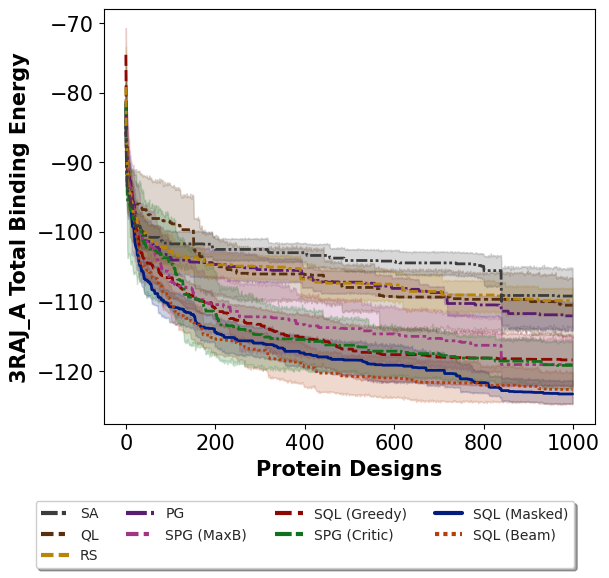}
    \caption{Convergence plots where the x axis indicated number of protein designs evaluated, and the y axis indicates the best energy (lower is better). We observe that across all tasks, SQL variants remain dominant, specifically SQL Masked which generated sequences non-sequentially performs best with the lowest variance.}
    \label{fig:pdb_all}
\end{figure}

In Table~\ref{fig:table_var} we report the standard deviation of all 10 seeds at various stages of optimisation for each task. In Figure~\ref{fig:pdb_all} we show line-plots for each task with 95\% confidence intervals.

\begin{figure}[ht!]
    \centering
    \includegraphics[width=\linewidth]{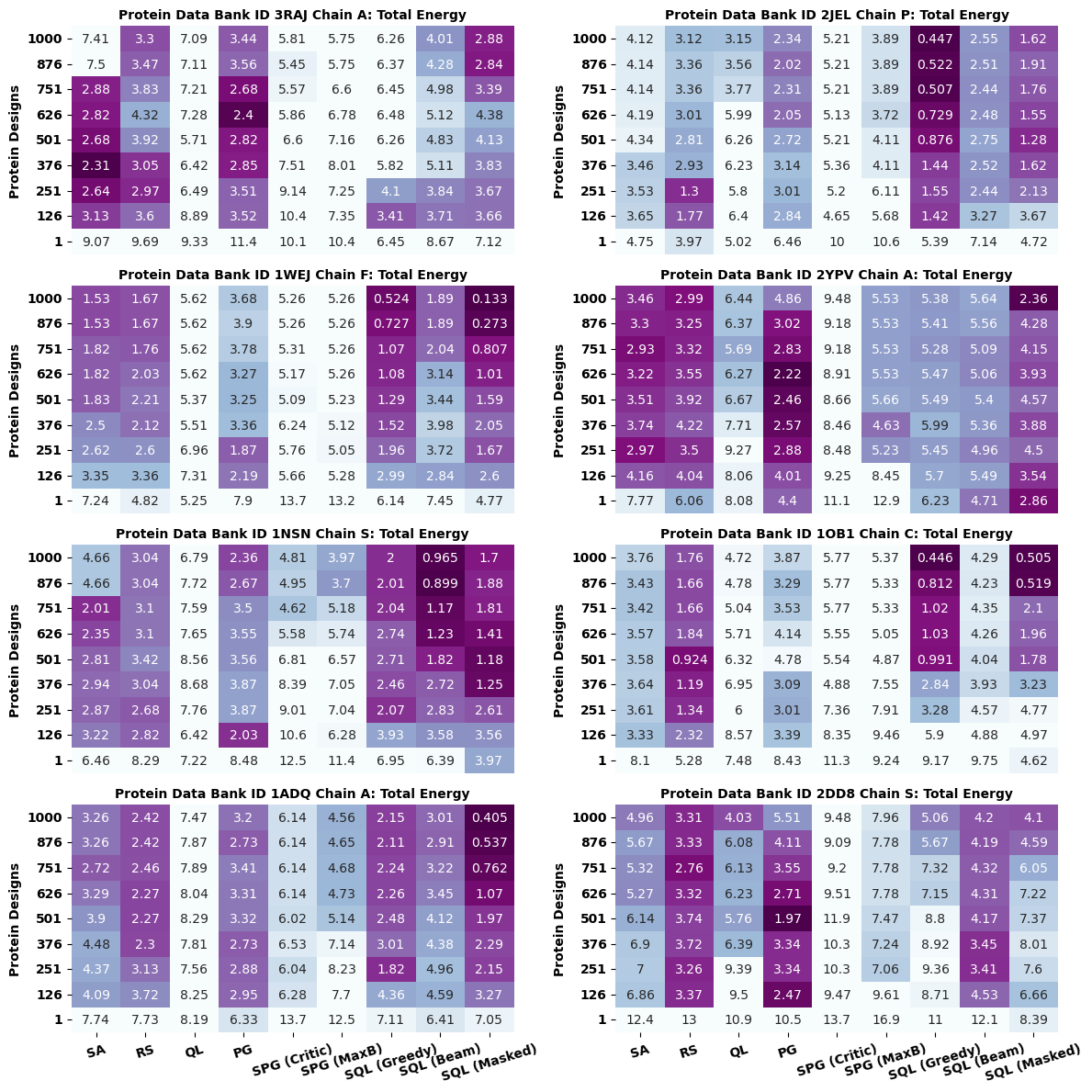}
    \caption{Plot of standard deviations for same intervals as \ref{fig:combtasks}}
    \label{fig:table_var}
\end{figure}

In Figure~\ref{fig:msa} we take the top sequences on PDB 1ADQ, Chain A protein optimization task. All but three of 300 sequences found with best energy scores $(-112.59)$ came from SQL, with two sequences from SPG and one from QL. We use multiple sequence alignment on sequences obtained from QL, SPG and SQL. 

\begin{figure}[th!]
    \centering
    \includegraphics[width=\linewidth,trim={1cm 1cm 3.1cm 0cm},clip]{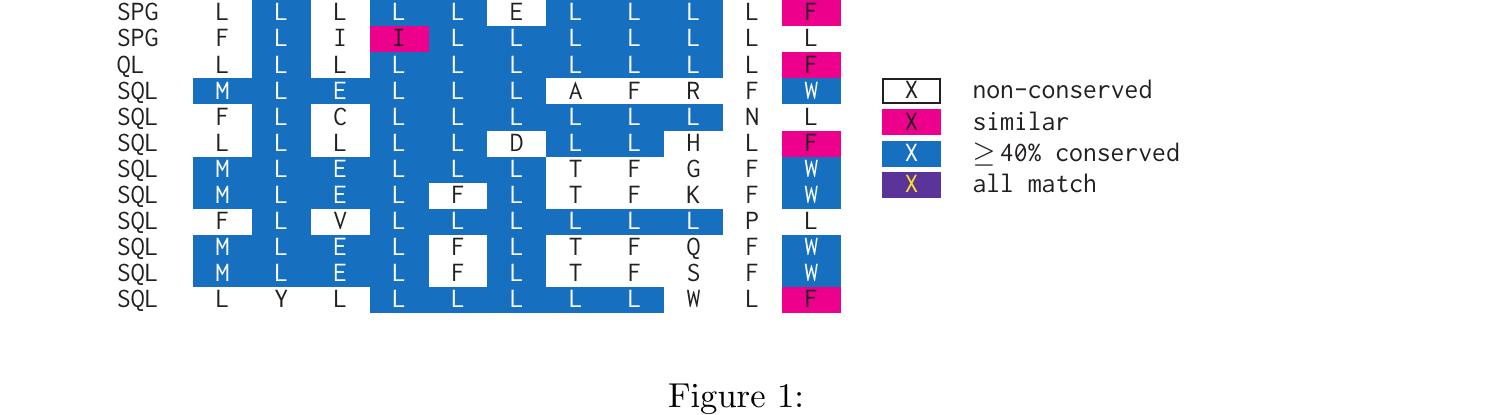}
    \caption{Multiple sequence alignment (MSA) analysis of top proteins found by all agents for the PDB 1ADQ Chain A interaction energy task. The only method that found these energy scoring proteins was SQL variants. Overall for PDB 1ADQ Chain A, SQL variants found over 300 unique proteins with the same objective function value (total energy) of $-112.59$. This MSA plot highlights the complexity of the problem and how so many, related, local minima exist.}
    \label{fig:msa}
\end{figure}

\section{Additional Algorithmic Details of SQL}\label{app:sql}
In this section we will give additional details of SQL. In Algorithm~\ref{alg:sql} we show the high level algorithm for SQL, which follows the same skeleton as Q-learning (also general off policy RL), however at each step is adapted for combinatorial optimization thus instead of outputting individual actions, each step outputs a full structure (e.g. a sequence of actions).

\begin{algorithm}[ht!]
\caption{Structured Q-learning}\label{alg:sql}
% \begin{algorithmic}
$\text{Set training steps } N$, $\text{ Set min buffer size } B$
% $\text{ , Set critic maximisation steps } M$\;
% $\text{Set Initial temperature } T_1^{\text{init}}$
$\text{ , \text{StructBuffer}} \gets []$\;
% $\text{Define a structure exploration operator } \Phi$\;
% $\text{Initalise structure critic/critics e.g.} \mathcal{S}(\cdot)$\;
\For{i $\text{in range}(N)$}{
    \eIf{len(\text{StructBuffer}) $\leq B$}{
        $\text{Sample, eval and store } \{\bms^{(i)},f(\bms^{(i)})\}_{i=0}^B$ 
    }{
        $\text{Train structure critics e.g. } \mathcal{S}(\cdot)$ \;
        $\bms^{\ast} = \arg \max_{\bms} \mathcal{S}(\bms)$ \;
        $\text{Get rand structure } \hat{\bms} \gets \Phi(\text{StructBuffer})$ \; 
        \eIf{$p(\textbf{accept}=\bms^{\ast})$  }{
            $\text{Exploit by evaluating structure } \bms^\ast \text{, add objective value } f(\bms^\ast) \text{ to } \text{StructBuffer}$
        }{
            $\text{Explore by evaluating structure } \hat{\bms} \text{, add objective value } f(\hat{\bms}) \text{ to } \text{StructBuffer}$
        }
    }
}
\end{algorithm}

\subsection{Neural Architectures}\label{app:arcs}
We will now give details of the neural architectures used. We use an embedding with 32 units to embed the amino acid's in proteins. We additionally use positional encoding on the embeddings. After a transformer is applied we feed this into an MLP in order to predict the logits (for PG/SPG) and targets (For QL and SQL) required for selecting next substructure. For all methods we use one transformer encoder block with embedding size 32, multi-head attention with 8 heads, 64 feed-forward units, layer norm, ReLU activation's and dropout probability 0.1. All methods therefore have the exact same number of parameters, making for fair algorithmic comparison. Both QL and SQL have an initial random exploration set to 32.

\subsubsection{Variable Size Structures}\label{app:vss}

% In general, many combinatorial optimization problems do not have a context observation, thus we must artificially define a starting observation by creating a token for $\emptyset$. 
% If there is context, one can simply use it instead of masking. Note, many combinatorial problems also do not give intermediate rewards, but rather only give this once the full structure has been evaluated, thus we must define the reward to be $0$ until the full structure has been determined, then we can evaluate this in the objective function $f$ to get the reward. 
For variable length sequences, we simply add an additional end of sequence action $[\textbf{EOS}]$ to actions which would stop generation and finalise that this is a complete structure ready for evaluation. For variable size structures we must also modify the reward function, such that it evaluates the structure in the objective function if the end of structure token $\textbf{EOS}$ is present.

\subsection{Variable Size Structure Exploration}\label{app:opp}
If one is dealing with variable length structures then we can add two additional operations \textit{shrink} \& \textit{expand}, choosing first which operation to apply, then applying the operation to the sampled structure $\bms^{(i)}$. Where \textit{shrink} deletes a substructure and expand adds a uniformly sampled substructure, at a uniformly chosen position in the structure. The other, more standard way to acquire a random structure, would be to just a random uniform set of variables. 

\section{Further Details of VAMP}\label{app:VAMP}

We will now present the proof of Theorem 1. 
% First we will detail how a combinatorial structure of variables can be represented, then we will show how to construct the same structure with VAMP.

% Where $[\cdot]$ defines a concatenation operator. 

% We can describe the VAMP transition function as below:
% \begin{equation*}
%  o_{t+1}^i = \begin{cases}
%  a_t & \text{ if } t=i \\
%  o_t^i & \text{ otherwise }
%  \end{cases}
% \end{equation*}

% % or in the vector form 
% We can also describe VAMP in vector form as $o_{t+1} = A_t o_t + b_t a_t$:
% \begin{equation}
%  \begin{aligned}
%     A_t &= \begin{pmatrix}
%     I_{t \times t} & 0 & 0_{t\times (L-t-1)} \\
%     0_{1 \times t} & 0 & 0_{1 \times (L-t-1)}\\
%     0_{(L-t-1)\times t} & 0 & I_{(L-t-1) \times (L-t-1)} 
%     \end{pmatrix}, 
%     b_t =  \begin{pmatrix}
%     0_{t} \\
%     1 \\
%     0_{L-t-1} 
%     \end{pmatrix}
%  \end{aligned}
% \end{equation}

\begin{thm}\label{thrm:1proof}
    For $\gamma_{\text{CO}}=1$ and any objective function $f$, the optimal policy $\pi^\ast$ picks a solution $s^\ast$ such that $f(s^\ast)  = \arg \max_{s \in\mathbb{S}} f(s)$.
    % finding the optimal policy $\pi^{\ast}$ is equivalent to finding the global optima $\bms^{\ast}$ of Equation~\ref{eq:combopt}.
\end{thm}
\begin{proof}
Using the identity $o_L=[a_i]_{0=1}^{L-1}=[x_i]_{i=0}^{L-1}=s$, we can show
$$ \pi^{\ast} = \arg \max_{\pi} \mathbb{E} \left[ \sum_{t=0}^{L-1} \gamma^t r(o_t,a_t,o_{t+1})  \right] = \arg \max_{\pi} \mathbb{E} \left[ f([ a_0, \ldots, a_{L-1}])  \right] = \arg \max_{s \in\mathbb{S}} f(s) = \bms^{\ast}
$$
\end{proof}

% Note that the policy is time-dependent (non-stationary), i.e., $\pi(t, o)$ since it assigns a value to $i$-th component of $o$. 

\section{Structured Policy Gradients}\label{sec:spg}

Structured Policy Gradients (SPG) aim to learn a policy $p(\pi)$ that will assign high probability to combinatorial variables that will improve the objective function, and low probabilities combinatorial variables that will not. As previously formulated in~\cite{BelloPLNB16,kool2018attention,NIPS2014_a14ac55a}, who define Structured Policy Gradients policy using the chain rule to factorize the combinatorial variables:

\begin{equation}\label{eq:fpolicy}
    \log p(s) = \log p([a_i]_{i=1}^L) = \sum_{i=1}^L \log p( a_i \mid [a_j]_{j=1}^i) 
\end{equation}

The authors of~\cite{kool2018attention} identify that for the TSP problem, it can be useful to set a greedy baseline $b(s)$. However, this greedy baseline equates to an additional objective function evaluation per parameter update. For many combinatorial optimization problems the objective function is extremely expensive, hence we are in need of a cheaper baseline. We propose to use the (cheaper) best objective function value observed so far as a greedy policy estimate. For some learning rate $\alpha_{k}$ at step $k$, we get the following update for policy parameters $\theta$;

\begin{equation}\label{eq:policyloss}
    \begin{array}{cc}
         \theta \Longleftarrow \theta + \alpha_k \cdot  \mathbb{E}_{\tau \sim \bm{\pi}_{\theta}} \left[ \nabla \log p(s) f(s) - b(s)\right] 
    \end{array}
\end{equation}

In our experiments, we compare SPG with both a learnt baseline (Critic), as done in ~\cite{BelloPLNB16} and greedy baseline (MaxB), as done in~\cite{kool2018attention}

\end{document}